\documentclass
{IEEEtran}
\usepackage{amsmath,amssymb,amsfonts}
\usepackage{xcolor}

\usepackage{comment}

\usepackage{booktabs}
\usepackage{graphicx}
\usepackage[normalem]{ulem}
\usepackage{tikz}
\usepackage{float}
\usepackage{xcolor}

\usetikzlibrary{decorations.pathmorphing}
\usepackage{algorithm}
\usepackage{algpseudocode}

\usepackage{subcaption}
\usepackage{times}
\usepackage{graphicx} 
\usepackage{multicol}    

\usetikzlibrary{shapes.geometric, arrows.meta, positioning}

\tikzset{
  block/.style   = {rectangle, draw, minimum width=3cm, minimum height=0.9cm,
                    align=center},
  startstop/.style = {ellipse, draw, minimum width=2.6cm, minimum height=0.9cm,
                      align=center},
  decision/.style = {diamond, draw, aspect=2, align=center, inner sep=1pt},
  line/.style    = {-{Stealth[length=2mm,width=2mm]}},
  node distance=1.4cm
}

  \makeatletter
   \let\NAT@parse\undefined
  \makeatother
  \usepackage[pdftex, pdfborderstyle={/S/U/W 0}]{hyperref}

\newcommand{\Z}{\mathbb{Z}}
\newcommand{\R}{\mathbb{R}}

\newcommand{\calD}{\mathcal{D}}

\newcommand{\calH}{\mathcal{H}}

\DeclareMathOperator*{\argmin}{arg\,min} 

\newtheorem{thm}{Theorem}
\newtheorem{rem}{Remark}
\newtheorem{lem}{Lemma}
\newtheorem{prop}{Proposition}
\newtheorem{cor}{Corollary}
\newtheorem{definition}{Definition}
\newtheorem{assumption}{Assumption}
\newtheorem{rec}{Requirement}

\newtheorem{nota}{Notation}

\newenvironment{proposition}{\begin{prop}}{\hfill \mbox{\small$\square$} \end{prop}}

\newenvironment{remark}{\begin{rem}}{\hfill $\bullet$ \end{rem}}

\newenvironment{defi}{\begin{definition}}{\hfill $\square$ \end{definition}}
\newenvironment{asp}{\begin{assumption}}{\hfill $\bullet$ \end{assumption}} 
\newenvironment{theorem}{\begin{thm}}{\hfill  \mbox{\small$\square$} \end{thm}}

\usepackage{soul}
\definecolor{darkgreen}{RGB}{0,100,0}
\definecolor{myred}{rgb}{0.75,0.2,0.2}
\definecolor{mydarkgreen}{rgb}{0.0, 0.6, 0.0} 
\definecolor{alcolor}{rgb}{.85,.27,.09} 

\definecolor{dark-red}{rgb}{0.4,0.15,0.15}
\definecolor{dark-blue}{rgb}{0.15,0.15,0.6}
\definecolor{medium-blue}{rgb}{0,0,0.5}
\definecolor{gray}{rgb}{0.4,0.4,0.4}

\long\def\invis#1{}

\begin{document}

\title{On Mobile Ad Hoc Networks for Coverage of Partially Observable Worlds}

\author{Edwin Meriaux$^*$\quad 
Shuo Wen$^*$\quad 
Louis-Roy Langevin\quad \\ Doina Precup\quad Antonio Lor\'{\i}a \quad Gregory Dudek

\thanks{* Co-first authors.}
\thanks{Edwin Meriaux, Shuo Wen, Doina Precup and Gregory Dudek are affiliated with the McGill University Computer Science Department and MILA Institute. (e-mail: edwin.meriaux@mail.mcgill.ca)}
\thanks{Edwin Meriaux is also affiliated with L2S at the University Paris-Saclay, CentraleSupélec.}
\thanks{Louis-Roy Langevin is affiliated with the McGill University Department of Mathematics and Statistics}
\thanks{Antonio Lor\'{\i}a is affiliated with CNRS L2S.}

}

\maketitle






\begin{abstract}
This paper addresses the movement and placement of mobile agents to establish a communication network in initially unknown environments. We cast the problem in a computational-geometric framework by relating the coverage problem and line-of-sight constraints to the Cooperative Guard Art Gallery Problem, and introduce its partially observable variant, the Partially Observable Cooperative Guard Art Gallery Problem (POCGAGP). We then present two algorithms that solve POCGAGP: CADENCE, a centralized planner that incrementally selects $270^\circ$ corners at which to deploy agents, and DADENCE, a decentralized scheme that coordinates agents using local information and lightweight messaging. Both approaches operate under partial observability and target simultaneous coverage and connectivity. We evaluate the methods in simulation across 1,500 test cases of varied size and structure, demonstrating consistent success in forming connected networks while covering and exploring unknown space. These results highlight the value of geometric abstractions for communication-driven exploration and show that decentralized policies are competitive with centralized performance while retaining scalability.
\end{abstract}


\IEEEpeerreviewmaketitle

\section{Introduction}

In this paper, we study how mobile agents can be positioned and coordinated to establish communication networks in initially unknown environments.

\subsection{Context and motivation}

Mobile Ad Hoc Networks (MANETs) consist of interconnected mobile devices (in the context of this paper, mobile sensor nodes) that are deployed dynamically without a predefined structure or topology and self-form in real time. 

\begin{figure}[]
    \centering
    \includegraphics[width=0.45\textwidth]{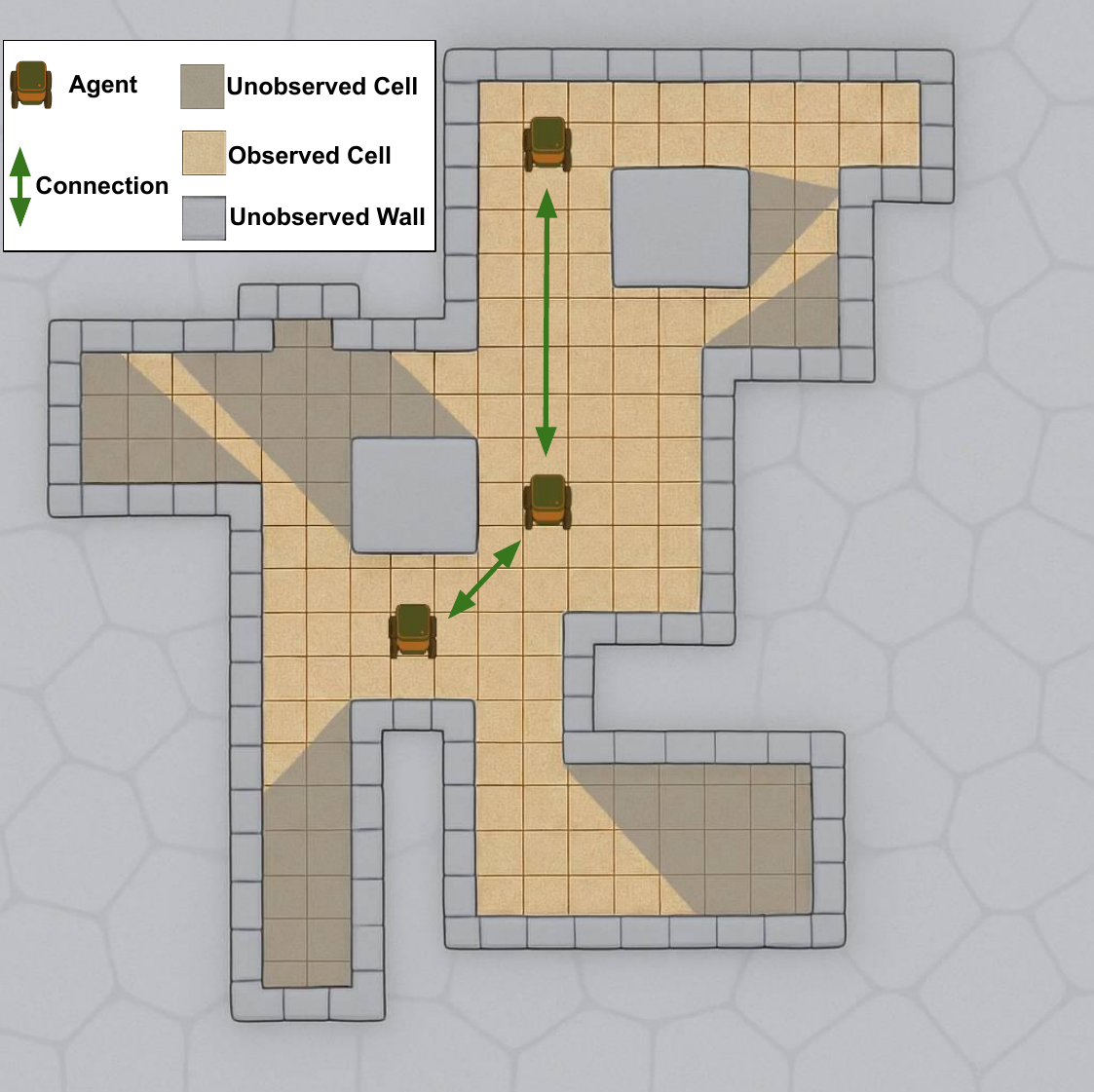} 
    \caption{Three robots forming a MANET in a partially observed indoor environment. The bright yellow areas represent observed cells, while the shadowed regions are uncovered. The two holes (polygons in the middle) and outer walls are physical barriers that cannot be covered. Green arrows indicate the agents' communication links.}
    \label{fig:deployed}
\end{figure}

MANETs comprise mobile sensor nodes that provide rapidly deployable, flexible, and resilient communication in environments where infrastructure is absent or temporary, such as disaster zones and remote regions~\cite{8991194}. For instance, MANETs have been successfully demonstrated in underground mines~\cite{zhang2023unmanned} or in confined-space search and rescue operations, such as in earthquake-damaged buildings~\cite{del2016survey}. The deployment of MANETs in such environments is complicated by physical barriers, such as dense rock formations~\cite{zhou2015rf}, metallic structures~\cite{aileen2021wifi}, or moving crowds, any of which can weaken or block wireless signals and limit node mobility. In disaster-affected or crowded areas, pre-existing maps often become inaccurate or obsolete due to structural collapse or blocked passageways. Corridors filled with debris, collapsed rooms, and large gatherings can drastically change the spatial layout such that the deployed nodes only have knowledge of the operational environment within their sensor range. These constraints give the MANET only partial information about the environment, making the setting \emph{partially observable}, at least initially, until the self-formed network evolves to achieve full coverage of the space. An example of a partially covered world is shown in Figure~\ref{fig:deployed}.

To ensure stable communication with users throughout the space, the MANET must simultaneously provide coverage and maintain network connectivity~\cite{wang2011coverage}. Users must be able to connect to any agent to transmit information packets, which may traverse multiple hops through intermediate nodes. Consequently, the system must guarantee that every relevant area falls within the communication range of at least one mobile node. Furthermore, the MANET must preserve connectivity \cite{poonawala_preserving_2017_7858679,ER_TAC_drones-sat} through single or multi-hop links, ensuring robustness against node movement, failures, or signal obstructions.

Broadly speaking, the network coverage problem concerns how effectively a set of agents can monitor, sense, or provide communication services over a physical environment~\cite{borralho2021survey}. The usual goal is to achieve complete coverage of the environment while maintaining communication between all agents, and simultaneously minimizing both the number of agents deployed and the total deployment time~\cite{ccabuk2021max,fan2010coverage}. As in other MANET settings, this challenge becomes significantly more complex in indoor, heavily obstructed environments, such as those constructed from concrete, rock, or metal, which severely degrade communication~\cite{meriaux2023test,aileen2021wifi}. In this setting, the connectivity aspect of the problem broadly concerns the ability of one agent to ``observe'' another. If two agents observe each other, they are connected. Depending on the sensing or communication modality, ``observing'' may mean, {\it e.g.}, that the other agent lies within the range of a LiDAR sensor, a wireless transceiver, or some other perception or communication device.

This paper addresses the problem of self-constructing MANETs in partially observable environments, ensuring both full spatial coverage and connectivity maintenance. 

\subsection{Coverage, connectivity, and the Art Gallery Problem(s)}

The network coverage problem~\cite{huang2003coverage} is closely related to a foundational question in computational geometry, known as {\it The Art Gallery Problem} (AGP) \cite{louvre-AGP}: How to place the fewest observers/cameras (referred to as guards) in a polygon, representing the surface of the world, so that every point inside is visible? Originally posed by Klee~\cite{honsberger1976mathematical}, the minimum-guard AGP seeks the smallest set of guards such that at least one can observe every interior point of the polygon.

Chvátal's Theorem~\cite{o1987art} establishes that $\left\lfloor \tfrac{n}{3} \right\rfloor$ guards are always sufficient to cover a polygon with $n$ vertices~\cite{chvatal1975combinatorial}. Subsequent works provided improved bounds on the number of agents that suffice to solve the AGP in orthogonal polygons, non-orthogonal worlds, and polygons with {\it holes} (polygons containing spaces not considered part of the world) which act as visibility barriers \cite{zylinski2004cooperative,zylinski2005cooperative}. Because these polygons model line-of-sight constraints, AGP is directly relevant to indoor sensor networks with LOS limitations.



The AGP was extended with the introduction of the {\it Cooperative Guards Art Gallery Problem} (CGAGP), which adds the requirement that agents must not only cover the entire world but also form a connected graph $G= (V,E)$. In this graph, the set of vertices $V$ represents the agents' locations, and the set of edges $E$ is defined such that an edge $e \in E$ exists if two agent locations (vertices) can observe each other. CGAGP naturally establishes a communication network among agents and can be applied to model scenarios such as multi-robot coordination, distributed sensing, and connected exploration. 

Determining the minimum number of guards in both the AGP and CGAGP settings is known to be NP-hard~\cite{1057165,liaw1993minimum}. 
Although computing an optimal solution is intractable in general, an upper bound on the required number of agents is known: 
for an orthogonal polygon with $n$ vertices and $h$ holes, the number of agents needed to achieve complete, connected 
coverage is bounded by~\cite{zylinski2004cooperative}:

\begin{equation}
M_{\mbox{\scriptsize CGAGP}} = \frac{n + 2h - 4}{2}.
\label{cgagp}
\end{equation}

This best worst-case bound~\cite{donald1999art} represents the smallest number of agents guaranteed to solve CGAGP in the most complex admissible polygon. In both AGP and CGAGP, computation is performed offline: the environment is fully known, and agent positions are fixed after deployment.

\subsection{Centralized and decentralized deployment algorithms}

MANET deployment algorithms for the connectivity and coverage problem generally fall into centralized or decentralized categories. Centralized approaches rely on a single entity with global knowledge to dictate agent actions, enabling efficient coordination and decision-making. However, they suffer from high computational complexity, limited scalability, and vulnerability to single points of failure. Classical centralized algorithms include \textit{Mixed Integer Linear Programming (MILP)} for optimal task allocation \cite{9515362}, and \textit{Centralized Model Predictive Control (MPC)} for multi-agent coordination in autonomous vehicle systems \cite{mihaly2020model}.

Conversely, decentralized control enables agents to make decisions using only local information and communications with peers~\cite{thien2018decentralized}. This approach provides two major benefits. First, the system becomes more robust, as the loss of a single agent does not prevent others from functioning independently. Second, computation is set such that each agent makes their own decision. this removes computational bottlenecks associated with one controller. Notable decentralized methods include \textit{Couzin's swarming model} for collective motion~\cite{couzin2002collective} and \textit{Ant Colony Optimization (ACO)} for path planning~\cite{ilie2010distributed}. Ultimately, the choice between centralized and decentralized depends on the specific computational and communication constraints of the system.

\subsection {Contributions}

Our contributions are summarized as follows.
We present the mathematical formulation of the Partially Observable Cooperative Guard Art Gallery Problem (POCGAGP). This novel formulation extends the CGAGP and AGP to the online setting  which requires that a MANET simultaneously explores, covers, and establishes connectivity within initially unknown environments. 

To address this challenge, we propose two MANET algorithms: Centralized Agent-Directed Exploration for Network Coverage (CADENCE) and Decentralized Agent-Directed Exploration for Network Coverage (DADENCE). CADENCE guarantees full coverage while maintaining the theoretical upper bounds of the classic CGAGP, whereas DADENCE leverages decentralized decision-making to maximize agent efficiency.

To rigorously validate these contributions, we introduce a new benchmark suite consisting of 1,500 orthogonal test worlds. We also introduce a novel quadtree decomposition complexity measure to rank the test worlds. In the extensive experiments, we demonstrate that our approaches outperform baseline methods in coverage capability, agent efficiency, and time-to-completion. Our results confirm that both algorithms effectively solve the POCGAGP while offering distinct advantages in theoretical guarantees and decentralized robustness, respectively.

\section{Problem Formulation}
\label{prob_formu}

\subsection{General problem setting}


\begin{defi}[Partially-observable CGAGP]\label{def:POCGAGP}
Given an initially unknown world $W \subset \mathbb{R}^N$ and a variable-size network of mobile agents, design an \emph{online} deployment policy that moves agents to positions $x \in W$ so that, using only information gathered so far, they progressively discover $W$ and eventually terminate in positions that allow them to collectively and cooperatively observe the whole world.
\end{defi}

We refer to the problem outlined above as the Partially-Observable Cooperative-Guard Art Gallery Problem (POCGAGP). Here, {\it collectively observing} the world means that the union of the regions observed by all agents covers the entire world. {\it Cooperatively observing} further requires that the deployed agents form a connected network. The notion of {\it observing} is made precise with Definition~\ref{defi:fov}.

The POCGAGP is a generalization of both the Art Gallery Problem (AGP) and the Cooperative Guard Art Gallery Problem (CGAGP). In the AGP, the goal is for guards (which we will use interchangeably with the term agent in the MANET setting of this paper) to collectively observe a \emph{known} world, without any requirement for cooperation (i.e., agents do not need to form a network). By contrast, in the CGAGP, agents must not only observe the world but also form a network while doing so, thus acting cooperatively.

For generality, Definition \ref{def:POCGAGP} does not limit the number of agents needed to collectively come to observe the entire world $W$ nor the time needed to achieve this objective. However, we impose the following constraints for this work:

\begin{rec}[Time Limit \(T_{\max}\)] \label{req_time}
The total deployment and agent movement process must occur over a finite time horizon of length \(T_{\max}\). 
\end{rec}

\begin{rec}[Agent Limit \(N_{\max}\)] \label{req_n}
The total number of agents that may exist in the world at any time is bounded by a given constant \(N_{\max}\).
\end{rec}

Requirement~\ref{req_time} is motivated by the fact that under an infinite time horizon, the POCGAGP degenerates into the CGAGP. If $T_{\max} = \infty$, the agents can first explore the entire (finite) world within an arbitrarily large time horizon, and then compute an optimal static configuration offline, deploy agents, and establish a network. However, such an approach is impractical for large-scale or time-sensitive deployments in the real world.

Requirement~\ref{req_n} captures a resource constraint by upper-bounding the number of deployable agents, reflecting real-world limitations like hardware, cost, or energy. Relaxing this constraint trivializes the POCGAGP, since under an unconstrained maximal number of deployable agents, one can place an agent at every position, $x \in W$, of the world.

These bounds remove the need for an algorithm to find a solution that minimizes the number of agents and the number of time steps. Seeking such an optimal solution would likely be NP-hard, as in the AGP and CGAGP. In this setting, we instead aim to find a solution that respects the resource constraints, including time, as solving an NP-hard optimization problem online is impractical. In addition, given insufficient resources, $T_{max}$ and $N_{max}$ can be set such that the world cannot be covered in the POCGAGP setting. An example would be an $N_{max} = 1$ to cover the world in  Figure~\ref{fig:deployed}. In those settings, there never exists a solution to the POCGAGP. We constrain ourselves to settings where a solution could exist.


Now that we have elaborated on Requirements~\ref{req_time} and \ref{req_n}, let us now clarify the graph connectivity requirements for POCAGP which follows from the same requirement in the CGAGP. To do this we introduce the following concept:

\begin{defi}[Visibility graph]\label{def:vgraph}
Let $W \subset \mathbb R^N$ be a world and $ A\subseteq W$ be a set of locations of deployed agents. The {\it visibility graph} of $ A$ in $W$, denoted $G( A)$, consists of a vertex set $V( A) =  A$ and edge set
    \begin{align*}
        E( A) := \{\, (x,y) \,:\, x,\, y  \in A \text{ and  $x$ is observable from $y$}\, \}.
    \end{align*}
\end{defi} 

  In Definition \ref{def:vgraph}, an edge exists between two points in the world if one is observable by the other, {\it i.e.,} a point is visible from another if a drawn line segment between them lies entirely within the polygon, unobstructed. We formally define the field of view in Definition \ref{FOV}.

To ensure that the agents form a valid communication network at all times, we impose the following requirement.

\begin{rec}[Connectivity of the visibility graph]\label{req_con}
The visibility graph formed by a network of agents that collectively observe a world $W$ must be connected at all times. 
\end{rec}

\subsection{The POCGAGP on the orthogonal plane}

In this paper, we restrict our analysis to worlds $W\subset \mathbb R^2$, consistent with the classical AGP and CGAGP formulations on which we build. Many practical deployment settings, such as building floors, tunnels, or mine shafts, are effectively planar, with agents operating on a single navigable layer. Focusing on 2D orthogonal worlds thus preserves the key geometric and communication challenges of coverage and connectivity while still enabling rigorous theoretical guarantees and tractable computation.

We start with formal definitions required to establish our algorithms. 

\begin{defi}[Interior angle]\label{def:angle}
 For a set $S \subseteq \R^2$, $x \in S$, and $0^\circ \le r \le 360^\circ$, $x$ is called an $r^\circ$ angle in $S$ if
    \[\lim_{\varepsilon \downarrow 0} \frac{\lambda[B_\varepsilon(x)\cap S]}{\pi\cdot\varepsilon^2} = \frac{r}{360},
\]        
    whenever this limit is well defined. $\lambda$ denotes the Lebesgue measure and
    $B_\varepsilon(x)$ is the ball of radius $\varepsilon$ centered at $x$.  We may say $r^\circ$ angle if $S$ is not ambiguous.
\end{defi}

\begin{defi}[Orthogonal polygon] \label{def:ortp}
An \emph{orthogonal polygon} is a subset \( P \subset \mathbb{R}^2 \) whose border consists of straight line segments (called walls) that are all parallel to the Cartesian axes. The walls being either parallel or perpendicular to each other, every interior angle of \( P \) is either \( 90^\circ \) or \( 270^\circ \). The intersecting points of the walls at these angles are referred to as {\it corners}.
\end{defi}

\begin{defi}[Orthogonal world] \label{def:ortw}
An \emph{orthogonal world} is an open subset \( W \subset \mathbb{R}^2 \) that consists of an open orthogonal polygon $P$, taken away $h\in \mathbb Z_{\geq 0}$ holes, and having $n$ corners. A hole $H$ corresponds to a closed orthogonal polygon in the world $W$, which we denote $H_h \subset \R^2$. $W := P \setminus \bigcup_{i \le h} H_i$. Additionally, it is required that $W$ is connected as a contiguous space; in particular, removing the holes does not disconnect the world into multiple components.
\end{defi}


In this orthogonal world \(W\), let \(A \subseteq W\) be the set of deployed agents, with cardinality \(|A|\), responsible for covering the world while maintaining a connected visibility graph.

\begin{remark}
    In the POCGAGP setting, agents may be added or removed from the set $A$. As agents cannot teleport into or out of the world, agents enter or leave the world as needed via a fixed deployment point \(x_d \in W\), provided that \(|A| \leq N_{\max}\) at all times. The point \(x_d\) is modeled as a fixed pseudo-agent representing the ability of the MANET of agents \(A\) to communicate with the outside world.
\end{remark}


This pseudo-agent will not be contained by the set $A$, but it counts as an agent which must remain connected by Requirement~\ref{req_con}. 
Connectivity between two agents in $A$, or between an agent $a \in A$ and the pseudo-agent, is defined as follows:

\begin{asp}[Mutual visibility]\label{asp1}
The agents are equipped with sensors such that for any two agents $a_x$, $a_y\in W$, it holds that if $a_x$ observes agent $a_y$, then agent $a_y$ observes agent $a_x$. \end{asp}

\begin{remark}
Under Assumption \ref{asp1}, the visibility graph $G(A\cup\{x_d\})$ is undirected and, by Requirement \ref{req_con}, connected. In more general settings, however, $G(A\cup\{x_d\})$ does not need to be undirected: the graph may be directed whenever visibility between two points is not mutual, for example, due to asymmetric sensor types or environmental conditions.
\end{remark}

Under Assumption~\ref{asp1}, the field of view of \(A \cup \{x_d\}\) in \(W\) is the set of points in \(W\) visible to all deployed entities, defined as follows.

\begin{defi}[Field of view]\label{FOV}
    Given a world $W \subset \R^2$ and a subset $A\subseteq W$, the field of view of $A$ in $W$ is defined as 
    \begin{align*}
        F_{W}(A) = \bigcup_{a \in A } \big\{ x \in W :\ 
        & tx + (1-t)a \in W 
        & \forall\, t \in (0,1) \big\}.
    \end{align*}
    We may simply write $F(A)$ if $W$ is clear from the context.
    \label{defi:fov}
\end{defi}

That is, $F_{W}(A \cup \{x_d\})$  consists of the points $x\in W$ that lie within the line of sight of at least one element of \(A \cup \{x_d\}\). Each agent of a MANET placed at the position \(x_i \in A\) ``observes'' its own visible region \(F_W(\{x_i\})\), for instance, the area sensed by an ideal $360^\circ$ LiDAR centered at $x_i$. Consequently, the collective field of view \(F_W(A \cup \{x_d\})\) represents the portion of the world currently known to the agents, as if they shared their LiDAR measurements with others.

Putting this all together, let us now define the exact POCGAGP setting we look to solve:
\begin{defi}\label{def:POCGAGP2} {\it (POCGAGP on the orthogonal plane, with maintained coverage and connectivity)}
  Given an orthogonal world $W\subset \mathbb R^2$, form a MANET having no more than a given number $N_{\max}$ of agents, a pseudo-agent $x_d \in W$ and, in $T_{\max}$ units of time, make the network's agents assume positions $A \subseteq W$ in a way that, as the network self-forms in real time, for every instant $t < T_{\max}$, the visibility graph $G(A \cup \{x_d\})$ is connected. In addition, prior to that instant, the set of agents, denoted $A^-$ must satisfy
\[
 F_W(A^- \cup \{x_d\}) \subseteq F_W(A \cup \{x_d\}).
\]
That is, both the connectivity and coverage must be maintained at any moment, and at the end the world is fully covered and connected.
\end{defi}

The requirement that both the connectivity and the coverage are maintained at all times (which is only imposed at the end state for CGAGP) has a clear practical motivation. For instance, there may exist a network user, either a human or a robot, and located anywhere within the world \(W\), who relies on the network formed by \(A \cup \{x_d\}\) to communicate with the outside world or with other users. Once such a user becomes connected, maintaining that connection is essential.

To ensure efficient and effective coverage of the finite world  \(W\),  we make the following assumption.  
\begin{asp}\label{asp2}
The orthogonal world $W$ is not convex. 
\end{asp}
Since $W$ is orthogonal, Assumption \ref{asp2} means that the world contains at least one $270^\circ$ interior angle. 

\section{MANET-building Algorithms}

\label{Algo}

In this section, we present two MANET algorithms to address POCGAGP. The first, called Centralized Agent-Directed Exploration for Network Coverage (CADENCE), uses a centralized controller, but is distributed in nature. The second, called  Decentralized Agent-Directed Exploration for Network Coverage (DADENCE), is fully decentralized.

\subsection{CADENCE}
\label{sec:cadence}

The CADENCE algorithm achieves full coverage of an orthogonal world $W \subset \mathbb R^2$ by iteratively constructing a MANET with an underlying connected visibility graph, through the systematic deployment of agents at observed, {\it valid} \(270^\circ\) corners, which are defined as follows. 

\begin{defi}[Valid corner]\label{valid}
The valid corners of an orthogonal world $W$ correspond to all the $270^\circ$ interior angles in $W$, except for the top-left corners of holes. The term {\it top-left} corner of a hole refers to its corner that has the largest $y$-coordinate, and taking the corner with the smallest $x$-coordinate in case of equalities.
\end{defi} 

An example of which corners are valid and which are not is shown in Figure~\ref{fig:all_270}, where red circles indicate the valid corners.

\begin{remark}
    The choice to treat all $270^\circ$ corners as valid except those at the top-left of holes is purely a design choice in our setting. One could equally well designate the top-left, bottom-left, top-right, or bottom-right $270^\circ$ corner of each hole as invalid, provided the choice is applied consistently.
\end{remark}

\begin{algorithm}[h!]
\caption{CADENCE}\label{alg:CADENCE}
    \begin{algorithmic}[1]
    \State \textbf{Input :}  an unknown world $W$, a blackbox $\alpha$ (representing the sensor observations of the agents) such that $\alpha(X)$ outputs the valid corners in $F_W(X)$ for any $X \subset W$, and an initial deployment point $x_d \in W$.
  
    \State Set $A = \emptyset$, $S' := \alpha(\{x_d\})$.
    \label{cadln:2}
    \While{$S' \ne A$ {\bf and}  $S' \ne A\cup \{x_d\}$} \label{cadln:3}
    \State \parbox[t]{\dimexpr\linewidth-\algorithmicindent}{%
        Pick $x \in  S' \setminus (A\cup \{x_d\}) $   
    }\label{cadln:4}
    \State \parbox[t]{\dimexpr\linewidth-\algorithmicindent}{%
        Set $A = A\cup\{x\}$ 
    }\label{cadln:5}
    \State \parbox[t]{\dimexpr\linewidth-\algorithmicindent}{%
        Set $S' := \alpha(A\cup \{x_d\})$ 
    }\label{cadln:6}
    \EndWhile\label{cadln:7}
    \State \parbox[t]{\dimexpr\linewidth-\algorithmicindent}{%
        Set $\mathcal{F} = F(A \cup\{x_d\})$.
    }\label{cadln:8}
    \State \textbf{Output :} $A$, $\mathcal{F}$.\label{cadln:9}
    \end{algorithmic}    
\end{algorithm}

The design of CADENCE is inspired by the fact that, under Assumptions~\ref{asp1} and \ref{asp2}, a world $W\subset \mathbb R^2$ can be fully covered with a connected visibility graph by placing vertices at all reflexive corners (interior angles $r^\circ$, with $r > 180$), which in orthogonal worlds are exactly the $270^\circ$ corners. Classical results in computational geometry show that such a reflexive vertex visibility graph yields shortest paths between any two points, even if they are not mutually visible~\cite{lavalle2006planning,tan2009visibility,kapoor1988efficient,de2000computational,dudek2024computational,latombe2012robot,o2017visibility}. As part of this discussion regarding visibility graphs, the previously defined term of {\it valid corner} will be extensively used to discuss the vertices in the visibility graph. For eventual use in CADENCE, those valid corners will refer to the position of deployed agents.

Consequently, such a visibility graph necessarily (i) covers the entire world, otherwise some points could not be reached, and (ii) is connected, since any shortest path between two non–mutually visible points can be represented as a sequence of its edges. Another significant consequence of placing the vertices of a visibility graph at all reflexive corners of an orthogonal polygon is that 
\begin{equation}
  \label{453} 
|A| \leq M_{\mbox{\scriptsize refl}}, \quad M_{\mbox{\scriptsize refl}} := \frac{n + 4h - 4}{2},
\end{equation}

$n$ is the number of $90^\circ$ or $270^\circ$ interior angles in $W$, and $h$ is the number of holes. By placing the vertices of the visibility graph only to all the reflexive corners, as above, $|A|$ does not exceed $M_{\mbox{\scriptsize refl}}$. We show this in the following statement.

\begin{lem}\label{cadence_lem_bad}
An orthogonal world $W$ contains exactly $M_{\mbox{\scriptsize refl}}$ reflexive corners.
\end{lem}
\begin{proof}
%
Let $\mathcal{H}$ be the set of holes in $W$. The corners in an orthogonal world consist of $n_{\text{outer}}$ corners on the outer boundary of $W$ and $n_H$ corners on each hole $H \in \mathcal{H}$. The outer boundary contains ($\frac{n_{\text{outer}}}{2} - 2$) $270^\circ$-corners, and since the boundary of each hole $H \in \mathcal{H}$ is oriented oppositely, a hole with $n_H$ vertices contributes $\frac{n_H}{2} + 2$ such corners. Summing over all $h$ holes, the total number of $270^\circ$-corners is
\[
\frac{n_{\text{outer}}}{2} - 2 + \sum_{H \in \calH}\left( \frac{n_H}{2} + 2 \right) = \frac{n + 4h - 4}{2}.
\]
\end{proof}

Given the goal of POCGAGP, it is not required that the visibility graph contain the shortest paths between two non mutually visible points. In other words, it is possible that not all \textit{reflexive corners} need to host agents. We then prove that we can reduce the bounds of $M_{refl}$ (Eq.~\ref{453}). This is done by removing one reflexive corner of each hole from the visibility graph (of all reflexive corners) without losing either connectivity or coverage.

\begin{theorem}\label{cadence_thm}
Under Assumptions~\ref{asp1} and \ref{asp2}, an orthogonal world $W$ is fully covered with a connected visibility graph by placing vertices at all valid corners, and the number of vertices is 
\begin{equation}
  \label{457} 
M_{\mbox{\scriptsize valid}} := \frac{n + 2h - 4}{2}. 
\end{equation}
\end{theorem}

\begin{proof}
Given an orthogonal world $W \subset \mathbb{R}^2$, we construct a modified copy $W'$ by adding arbitrarily thin walls (of negligible width). Each added wall has one endpoint at an invalid (top-left corner of the hole) $270^\circ$ corner and extends into free space until it meets either the outer boundary of $W$ without intersecting any hole, or the boundary of a hole. By convention, all the added walls are assumed to be either horizontal or vertical to prevent creating disconnected spaces in $W'$.

This transformation does not disconnect $W'$, since for any two points in $W'$, a continuous path still exists. This transformation effectively reduces the number of necessary reflexive corners by exactly one per hole. This action can only reduce the set of reflexive angles, making the set of reflexive angles in $W'$ a subset of those in $W$. This construction is illustrated in Figures~\ref{fig:all_270}-\ref{fig:valid}: starting from the world where agents are deployed at all reflexive corners (Figure~\ref{fig:all_270}), one hole is connected to the outer wall (Figure~\ref{fig:connect_1}), then holes are connected (Figure~\ref{fig:connect_2}), until a hole-free world remains (Figure~\ref{fig:valid}). To note, once a hole is connected to part of the outer boundary, it is no longer a hole but part of the outer boundary.

The transformed $W'$ no longer contains the holes of $W$, and thus all of its reflexive corners are valid corners. It must be noted that $W'$ might have fewer valid corners than $W$, as in certain cases the thin wall construction removes a valid angle, as seen in Figure~\ref{fig:connect_2}. This construction of thin walls never creates a new valid corner. Given the set of valid corners in $W'$, there exists a connected visibility graph that covers the full world with at most the set of valid corners. By definition, the number of valid corners is 
\[M_{\mbox{\scriptsize valid}}= M_{\mbox{\scriptsize refl}}- h= \frac{n + 2h - 4}{2}. \]

\end{proof}

\begin{proposition}\label{cadence_prop}
Given an orthogonal world $W\subset \mathbb R^2$ satisfying Assumptions~\ref{asp1} and \ref{asp2}, the algorithm CADENCE forms a MANET having no more than $M_{\mbox{\scriptsize valid}}$ agents, in finite time, and always halts having covered the whole world. Moreover, at any time during the execution of CADENCE, the visibility graph $G(A \cup \{x_d\})$ is connected and the world coverage is maintained.
\end{proposition}

\begin{proof}
After Theorem \ref{cadence_thm}, for an orthogonal world satisfying Assumptions~\ref{asp1} and \ref{asp2} to be fully covered, by no more than $M_{\mbox{\scriptsize valid}}$ agents, it suffices to deploy agents to all the valid corners. We show that the CADENCE algorithm does precisely that, and it halts only after placing agents at all valid corners. 

CADENCE is initialized with no agents deployed, i.e., with $A = \emptyset$. Therefore, initially, the only portion of the world that is covered corresponds to what is visible from the deployment point $x_d$; this corresponds to $\mathcal{F}:=F(\{x_d\})$. Given Assumption \ref{asp2}, there exists an initially unknown set of valid corners $S \neq \emptyset$.

Until the algorithm halts it is only pertinent to assume that at least one valid corner is not occupied, because otherwise, $F(A \cup\{x_d\}) = F(G)$, where $G$ is the set of vertices in the connected visibility graph generated by Theorem \ref{cadence_thm}, making  $F(A \cup\{x_d\})= W$. Now, assume that among the unoccupied valid corners, at least one is visible from points $x \in A \cup\{x_d\}$. Therefore, $S'\neq A$ and $S' \ne (A\cup \{x_d\})$ and, consequently, the \texttt{while} loop at {\it lines~\ref{cadln:3}--\ref{cadln:7}} is executed until $S'= A$ or $S' = (A\cup \{x_d\})$; that is, until all the visible corners are successively occupied by an agent or the deployment point. Then, either no valid corner remains unoccupied and, by Theorem \ref{cadence_thm}, $F(A \cup\{x_d\}) = W$, or all of the remaining unoccupied corners are hidden from points $x \in A \cup\{x_d\}$. 

In this latter case, the set of hidden corners is $ S \setminus S'$, so let $x^h\in  S \setminus S' $ denote the position of a valid corner hidden from any point $x \in A \cup\{x_d\}$. By the mutual-visibility hypothesis (Assumption \ref{asp1}), the positions $x \in A \cup\{x_d\}$ are also hidden from $x^h\in S \setminus S' $. This forms a contradiction. By Theorem \ref{cadence_thm}, there exists a visibility graph $G$ such that $F(G)=W$, and $G$ is composed of valid corners which necessarily contain vertices at the positions $x^h$. From the latter, the fact that $F(G)=W$, and that the positions  $x \in A \cup\{x_d\}$ are hidden from any $x^h\in  S \setminus S' $, the only way this is possible (given theorem 1) is that $\{A \cup\{x_d\}\}\cap W =\emptyset$, which itself is impossible. Thus, no unoccupied valid corner remains hidden, and upon termination, CADENCE forms a graph $G$ that covers the entire world $W$. Since $|G| = M_{\text{valid}}$, the CADENCE result follows.

Finally, note that as CADENCE runs as described above, every time that {\it line~\ref{cadln:5}}  is executed, the actual visibility graph remains connected (in view of Assumption \ref{asp2}) and the actual partial coverage of the world is maintained or expanded. 
\end{proof}

Given this bound, the key strength of CADENCE is that it is agnostic to prior knowledge of the world: in the online POCGAGP, it can deploy agents to form a connected MANET that fully covers the world while achieving the same worst-case upper bound on the number of agents as if the deployment were planned in a fully known environment. This means that CADENCE performs equally well in both POCGAGP and CGAGP settings.

\begin{figure}
    \centering
    \begin{subfigure}[t]{0.45\columnwidth}
        \centering
        \begin{tikzpicture}
            \draw[thick] (0.5,0) -- (0.5,3) -- (4,3) -- (4,0) -- (0.5,0);
            \fill (1,0) -- (2,0) -- (2,0.5) -- (1,0.5) -- cycle;
            \fill (1,1) -- (2,1) -- (2,2.5) -- (1.5,2.5) -- (1.5,1.5) -- (1,1.5) -- cycle;
            \fill (2.5,1) -- (3.5,1) -- (3.5,2.5) -- (2.5,2.5) -- (2.5,2) -- (3,2) -- (3,1.5) -- (2.5,1.5) -- cycle;
            \fill (2.5,0.75) -- (3,0.75) -- (3,0.25) -- (2.5,0.25) -- cycle;

            \draw[fill=white, draw=white] (1.035,-0.1) -- (1.965,-0.1) -- (1.965,0.465) -- (1.035,0.465) -- cycle;

            \fill[red] (1,1) circle (0.1cm);
            \fill[red] (2,1) circle (0.1cm);
            \fill[red] (2,2.5) circle (0.1cm);
            \fill[red] (1,1.5) circle (0.1cm);
            \fill[red] (2.5,1) circle (0.1cm);
            \fill[red] (2.5,1.5) circle (0.1cm);
            \fill[red] (2.5,2) circle (0.1cm);
            \fill[red] (3.5,1) circle (0.1cm);
            \fill[red] (3.5,2.5) circle (0.1cm);
            \fill[darkgreen] (1.5,2.5) circle (0.1cm);
            \fill[darkgreen] (2.5,2.5) circle (0.1cm);
            
            \fill[darkgreen] (2.5,0.75) circle (0.1cm);
            \fill[red] (3,0.75) circle (0.1cm);
            \fill[red] (3,0.25) circle (0.1cm);
            \fill[red] (2.5,0.25) circle (0.1cm);
            
            \fill[red] (2,.5) circle (0.1cm);
            \fill[red] (1,.5) circle (0.1cm);

        \end{tikzpicture}
        \subcaption{Original polygon $W$.}
        \label{fig:all_270}
    \end{subfigure}
    \hfill
    \begin{subfigure}[t]{0.45\columnwidth}
        \centering
        \begin{tikzpicture}
            \draw[thick] (0.5,0) -- (0.5,3) -- (4,3) -- (4,0) -- (0.5,0);
            \fill (1,0) -- (2,0) -- (2,0.5) -- (1,0.5) -- cycle;
            \draw[fill=white, draw=white] (1.035,-0.1) -- (1.965,-0.1) -- (1.965,0.465) -- (1.035,0.465) -- cycle;
            \fill (1,1) -- (2,1) -- (2,2.5) -- (1.5,2.5) -- (1.5,1.5) -- (1,1.5) -- cycle;
            \fill (2.5,1) -- (3.5,1) -- (3.5,2.5) -- (2.5,2.5) -- (2.5,2) -- (3,2) -- (3,1.5) -- (2.5,1.5) -- cycle;
            \fill (2.5,0.75) -- (3,0.75) -- (3,0.25) -- (2.5,0.25) -- cycle;

            \fill[blue] (1.5,2.5) -- (1.5,3) -- (1.6,3) -- (1.6,2.5) -- cycle;
            
            \fill[red] (1,1) circle (0.1cm);
            \fill[red] (2,1) circle (0.1cm);
            \fill[red] (2,2.5) circle (0.1cm);
            \fill[red] (1,1.5) circle (0.1cm);
            \fill[red] (2.5,1) circle (0.1cm);
            \fill[red] (2.5,1.5) circle (0.1cm);
            \fill[red] (2.5,2) circle (0.1cm);
            \fill[red] (3.5,1) circle (0.1cm);
            \fill[red] (3.5,2.5) circle (0.1cm);
            \fill[darkgreen] (2.5,2.5) circle (0.1cm);
            
            \fill[darkgreen] (2.5,0.75) circle (0.1cm);
            \fill[red] (3,0.75) circle (0.1cm);
            \fill[red] (3,0.25) circle (0.1cm);
            \fill[red] (2.5,0.25) circle (0.1cm);
            
            \fill[red] (2,.5) circle (0.1cm);
            \fill[red] (1,.5) circle (0.1cm);
        \end{tikzpicture}
        \subcaption{Reducing one hole by connecting one invalid reflexive corner to the outer wall.}
        \label{fig:connect_1}
    \end{subfigure}

    \par\bigskip

    \begin{subfigure}[t]{0.45\columnwidth}
        \centering
        \begin{tikzpicture}
            \draw[thick] (0.5,0) -- (0.5,3) -- (4,3) -- (4,0) -- (0.5,0);
            \fill (1,0) -- (2,0) -- (2,0.5) -- (1,0.5) -- cycle;
            \draw[fill=white, draw=white] (1.035,-0.1) -- (1.965,-0.1) -- (1.965,0.465) -- (1.035,0.465) -- cycle;
            \fill (1,1) -- (2,1) -- (2,2.5) -- (1.5,2.5) -- (1.5,1.5) -- (1,1.5) -- cycle;
            \fill (2.5,1) -- (3.5,1) -- (3.5,2.5) -- (2.5,2.5) -- (2.5,2) -- (3,2) -- (3,1.5) -- (2.5,1.5) -- cycle;
            \fill (2.5,0.75) -- (3,0.75) -- (3,0.25) -- (2.5,0.25) -- cycle;

            \fill[blue] (1.5,2.5) -- (1.5,3) -- (1.6,3) -- (1.6,2.5) -- cycle;
            \fill[blue] (2.5,0.75) -- (2.5,1) -- (2.6,1) -- (2.6,0.75) -- cycle;
            
            \fill[red] (1,1) circle (0.1cm);
            \fill[red] (2,1) circle (0.1cm);
            \fill[red] (2,2.5) circle (0.1cm);
            \fill[red] (1,1.5) circle (0.1cm);
            \fill[red] (2.5,1.5) circle (0.1cm);
            \fill[red] (2.5,2) circle (0.1cm);
            \fill[red] (3.5,1) circle (0.1cm);
            \fill[red] (3.5,2.5) circle (0.1cm);
            \fill[darkgreen] (2.5,2.5) circle (0.1cm);
            
            \fill[red] (3,0.75) circle (0.1cm);
            \fill[red] (3,0.25) circle (0.1cm);
            \fill[red] (2.5,0.25) circle (0.1cm);
            
            \fill[red] (2,.5) circle (0.1cm);
            \fill[red] (1,.5) circle (0.1cm);
        \end{tikzpicture}
        \subcaption{Connecting a hole to another hole.}
        \label{fig:connect_2}
    \end{subfigure}
    \hfill
    \begin{subfigure}[t]{0.45\columnwidth}
        \centering
        \begin{tikzpicture}
            \draw[thick] (0.5,0) -- (0.5,3) -- (4,3) -- (4,0) -- (0.5,0);
            \fill (1,0) -- (2,0) -- (2,0.5) -- (1,0.5) -- cycle;
            \draw[fill=white, draw=white] (1.035,-0.1) -- (1.965,-0.1) -- (1.965,0.465) -- (1.035,0.465) -- cycle;
            \fill (1,1) -- (2,1) -- (2,2.5) -- (1.5,2.5) -- (1.5,1.5) -- (1,1.5) -- cycle;
            \fill (2.5,1) -- (3.5,1) -- (3.5,2.5) -- (2.5,2.5) -- (2.5,2) -- (3,2) -- (3,1.5) -- (2.5,1.5) -- cycle;
            \fill (2.5,0.75) -- (3,0.75) -- (3,0.25) -- (2.5,0.25) -- cycle;

            \fill[blue] (2.5,2.5) -- (2.5,3) -- (2.6,3) -- (2.6,2.5) -- cycle;
            \fill[blue] (1.5,2.5) -- (1.5,3) -- (1.6,3) -- (1.6,2.5) -- cycle;
            \fill[blue] (2.5,0.75) -- (2.5,1) -- (2.6,1) -- (2.6,0.75) -- cycle;
            
            \fill[red] (1,1) circle (0.1cm);
            \fill[red] (2,1) circle (0.1cm);
            \fill[red] (2,2.5) circle (0.1cm);
            \fill[red] (1,1.5) circle (0.1cm);
            \fill[red] (2.5,1.5) circle (0.1cm);
            \fill[red] (2.5,2) circle (0.1cm);
            \fill[red] (3.5,1) circle (0.1cm);
            \fill[red] (3.5,2.5) circle (0.1cm);
            
            \fill[red] (3,0.75) circle (0.1cm);
            \fill[red] (3,0.25) circle (0.1cm);
            \fill[red] (2.5,0.25) circle (0.1cm);
            
            \fill[red] (2,.5) circle (0.1cm);
            \fill[red] (1,.5) circle (0.1cm);
        \end{tikzpicture}
        \subcaption{Formation of the full polygon $W'$ with no holes.}
        \label{fig:valid}
    \end{subfigure}

    \caption{Transformation of the polygon $W$ into $W'$ to remove all invalid reflexive corners (in green) while retaining only the valid corners (in red). The added walls for $W'$ are in dark blue.}
    \label{fig:polygon_transform}
\end{figure}
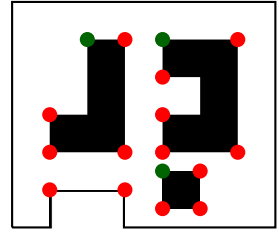
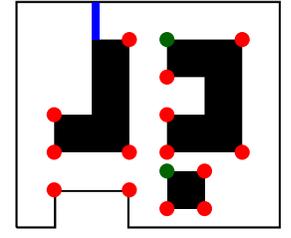
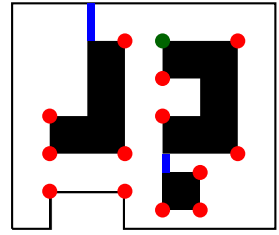
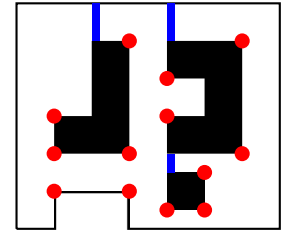

\subsection{DADENCE}
\label{sec:dadence}
Our second algorithm is decentralized: a set of agents are sequentially spawned in at an initial point $x_d$ in an orthogonal polygon $W$ and subsequently moves to strategically chosen positions.
More precisely, at a given instant in the process, a target point $x^*$ at the edge of the field of view (we will define as the border in Definition \ref{def:border}) 
is selected. The point $x^*$ is selected such that the average distance to all agents and the deployment point is minimized. Once the $x^*$ is selected, the agents advance towards it in steps of fixed length until the agents can observe beyond $x^*$ ({\it i.e.,} until $x^*$ is no longer at the edge of the field of view). Because the steps by which the agents advance are of fixed length, it is natural to consider that the world is a connected subset of $\mathbb{Z}^2$ but not $\mathbb R^2$ ({\it i.e.,} a discretized world). We stress that this entails no loss of generality, since the step size is arbitrary and can be chosen arbitrarily small.

We then introduce the following concepts, along with the discrete world.
\begin{defi}[Real embedding]
Given a subset in a discrete world $A \subset \Z^2$, the real embedding of $A$ is defined as 
    \begin{align*}
        \R(A) := \bigcup_{s \in A} \left\{x \in \R^2 \,:\, \max(| x_1 - s_1 |,|x_2-s_2|) \le \tfrac{1}{2}\right\}\,.
    \end{align*}
\end{defi}
In particular, for a given point $x \in \Z^2$, its real embedding, denoted $\mathbb R(\{x\})$, corresponds to a square region around that point and is called a {\it cell}.\label{cell}

\begin{defi}[Discrete field of view]
    \label{WFOV}
  Let $W$ be a discrete world and $A$ be a subset of $W$. For $x \in W$, let 
    \begin{align*}
        \calD(x) := \{y \in W \,:\, \R(\{y\})\cap F_{\R(W)}(\{x\}) \ne \emptyset\}, 
    \end{align*}
where $F_{(\cdot)}(\cdot)$ is as in Definition \ref{FOV}. Then, the discrete field of view of $A$ in $W$ is defined as
    \begin{align*}
        \tilde{F}_W(A) = \bigcup_{x \in A}\{y \in W \,:\, y \in \calD(\{x\}) \text{ or } x \in \calD(\{y\})\}.\
    \end{align*}
We may simply write $\tilde{F}(A)$ if $W$ is not ambiguous.
\end{defi}

Similar to the continuous worlds, two points $x_a$ and $x_b$ of $W$ are in each other's discrete field of view if there exists an unobstructed straight line in $\R(W)$ between $x_a$ and a point in the cell $\R(\{x_b\})$ or there exists an unobstructed straight line between $x_b$ and a point in the cell $\R(\{x_a\})$.


Note that in a discretized world, agents can move towards their given target $x^*$, advancing only in the North, South, East, or West directions. That is, along paths from one cell to another that are adjacent. This leads to the following concept.

\begin{defi}[Square graph]
\label{square_graph}
  Given $W \subset \Z^2$, the square graph of $W$, denoted $\bar{G}(W)$, has vertex set $\bar{V}(W) = W$ and edge set $\bar{E}(W)$ consisting of all pairs of vertices $u$ and $v$ with
    \begin{align*}
        |u_1 - v_1| + |u_2 - v_2| = 1\,.
    \end{align*}
\end{defi}

Then, for the agents to move from their location towards a point $x^*$ in the border of the discrete field of view optimally, they follow a shortest path on the square graph through strictly closer neighbors. We will now introduce these concepts:

\begin{nota}[Shortest path]
For two vertices $u,v \in \bar{V}(W)$, we denote the length of a shortest path on the graph $\bar{G}(W)$, from $u$ to $v$, as $L_{u,v}$ .
\end{nota}

\begin{defi}[Strictly closer neighbors]
    For two vertices $u,v \in V$, we denote as
    \begin{align*}
        N_v(u) := \{w \in \bar{V}(W) \,:\, (u,w) \in \bar{E}(W),\, L_{w,v} < L_{u,v}\}, 
    \end{align*}
the set of neighbors of $u$ that belong to shortest paths connecting $u$ and $v$. 
\end{defi}

\begin{defi}[Border]
\label{def:border}
 Given a partially observed world $F \subset W$, we define the border of $F$ within $W$ as
    \begin{align*}
        B_W(F) := \{x \in F \,:\, x \text{ has a neighbor } y \in W \text{ where } y \notin F\}. \\
    \end{align*}
Above, neighbors refer to the adjacent vertices in the square graph $\bar{G}(W)$. We may simply write $B(F)$ if $W$ is not ambiguous.
\end{defi}

We remark that, under Definitions~\ref{WFOV} and~\ref{def:border}, an exception arises when two cells intersect only at a corner (Figure~\ref{fig:forbidden-diagonal}). In this configuration, a line segment between the cell centers passes through the shared corner. As a result, the field of view incorrectly includes a region of $W$ for which no valid path exists for any agent to reach that border point (per Definition~\ref{square_graph}), creating an artificial and unrealistic “see-through-the-corner” effect. To rule out this undesired case, we impose the following assumption.
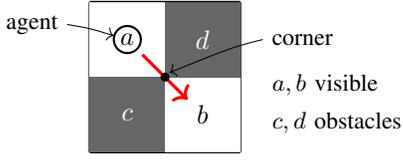
\begin{figure}[t]
  \centering
  \begin{tikzpicture}[scale=1]

    \foreach \x in {0,1,2} {
      \draw[thick] (\x,0) -- (\x,2);
    }
    \foreach \y in {0,1,2} {
      \draw[thick] (0,\y) -- (2,\y);
    }

    \fill[white!15] (0,1) rectangle (1,2);  
    \fill[white!15] (1,0) rectangle (2,1);  

    \fill[black!60] (1,1) rectangle (2,2);  
    \fill[black!60] (0,0) rectangle (1,1);  

    \node at (0.5,1.5) {$a$};
    \node[white] at (1.5,1.5) {$d$};
    \node[white] at (0.5,0.5) {$c$};
    \node at (1.5,0.5) {$b$};

    \draw[thick] (0.5,1.5) circle (0.18);

    \draw[red,very thick,->] (0.7,1.3) -- (1.3,0.7);

    \fill[black] (1,1) circle (0.06);

    \node[anchor=west] (cornerlabel) at (2.3,1.5) {\small corner};
    \draw[->] (cornerlabel.west) .. controls (1.9,1.4) .. (1.06,1.06);

    \node[anchor=east] (agentlabel) at (-0.3,1.7) {\small agent};
    \draw[->] (agentlabel.east) .. controls (0.0,1.65) .. (0.32,1.55);

    \node[anchor=west,align=left] at (2.3,0.9) {\small $a,b$ visible};
    \node[anchor=west,align=left] at (2.3,0.4) {\small $c,d$ obstacles};

  \end{tikzpicture}
  \caption{Forbidden diagonal line-of-sight: Assumption~\ref{asp3} rules out configurations where a top-left cell sees a bottom-right cell through a single diagonal corner between two obstacle cells (cf. Definition~\ref{WFOV}).}
  \label{fig:forbidden-diagonal}
\end{figure}

\begin{asp}
Consider any $2\times 2$-portion of the square graph $\bar G(W)$ as seen in Figure~\ref{fig:forbidden-diagonal}, formed by cells $a$, $b$, $c$, and $d$.  If $a, \, b \in W$, then, necessarily either $c \in W$ or $d\in W$. Similarly, if $c, \, d \in W$, then, necessarily either $a \in W$ or $b\in W$.
\label{asp3}
\end{asp}

Equivalently, two cells in $W$ may share a corner only if they are both adjacent to the same cell.

As the agents move towards the target $x^*$ as explained previously, they must form a connected visibility graph, which is defined as follows. 

\begin{defi}[Discrete visibility graph]
  Let $W$ be a world and $A$ a finite subset of $W$. The discrete visibility graph  of $A$ in $W$, denoted $\tilde{G}(A)$, has vertex set $\tilde{V}(A) = A$ and edge set
    \begin{align*}
        \tilde E(A) = \{xy \,:\, x,y \in A \text{ and } x \in \tilde{F}(\{y\})\}.
    \end{align*}
\end{defi}

\begin{algorithm}
\caption{DADENCE}\label{algo:dadence}
\begin{algorithmic}[1]
\State \textbf{Input:} an unknown world $W$, a blackbox $\beta$ (representing the sensor observations of the agents) such that $\beta(X) = F_W(X)$ for any $X \subset W$, and an initial deployment point $x_d \in W$.
  \State  Set $A = \emptyset$ and  $\mathcal{F} = \beta(\{x_d\})$.\label{832}
  \While{$B(\mathcal F) \neq \emptyset$}\label{833,37}
   \If{ $A = \emptyset$ } $A=\{x_d\}$ \label{833} \EndIf
    \State $x^* = \argmin_{x \in B(\mathcal{F})} \sum_{s \in A\cup \{x_d\}} L_{x,s}$ \label{dadln:6}
    
    \If{$x_d \in A$} \label{835}
      \State set $x = x_d$\label{1046} 
      \While{$x \in A$}\label{wh:837} 
        \State pick $y \in N_{x^*}(\{x\})$ and set $x = y$
      \EndWhile\label{wh:839}
      \State $A = A \cup \{x\} \setminus \{x_d\}$ \label{840}
      \State $\mathcal{F} = \mathcal{F} \cup \beta(\{x\})$\label{1051} 
    \Else\label{el:842}
      
      \State $A' = A$; \quad $\mathcal{S} = \emptyset$ \label{1054}
      \ForAll{$s \in A$}\label{for:844}
        \If{$N_{x^*}(s)\setminus A' \neq \emptyset$}
          \State $\mathcal{S} = \mathcal{S} \cup \{s\}$
          \State pick $x_s \in N_{x^*}(s)\setminus A'$
          \State $A' = A' \cup \{x_s\} \setminus \{s\}$
        \EndIf
      \EndFor \label{1061}
      \While{$\tilde G(A'\cup\{x_d\})$ is disconnected \label{wh:851}
        \\  \hspace{18mm} 
        \textbf{ or } $\mathcal{F} \not\subseteq \beta(A'\cup\{x_d\})$ 
         } 
        \State pick $s \in \mathcal{S}\setminus A' $ 
        \State $A' = A' \cup \{s\} \setminus \{x_s\}$
        \State $\mathcal{S} = \mathcal{S} \setminus \{s\}$
      \EndWhile\label{858}
      \If{$A = A'$} \label{1071}
        \State $A = A \cup \{x_d\}$ \label{860}
      \Else
        \State $A = A'$ \label{862}
      \EndIf
      \State $\mathcal{F} = \beta(A\cup\{x_d\})$\label{864}
    \EndIf
  \EndWhile\label{867}
  \State \textbf{Output:} $A$
\end{algorithmic}
\end{algorithm}

\begin{figure}[h!]
  \centering
  \includegraphics[width=0.4\textwidth]{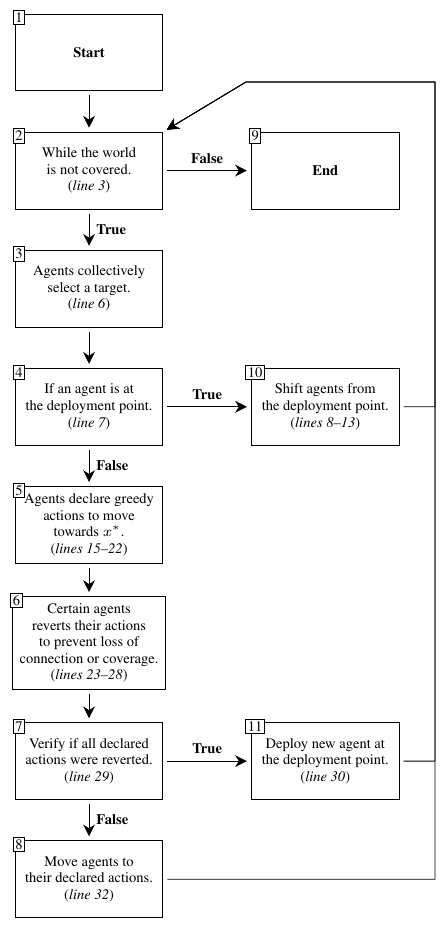}
  \caption{Flowchart illustrating the DADENCE algorithm. The numbered blocks (1--11) represent its main components.}
  \label{flowchart}
\end{figure}



    





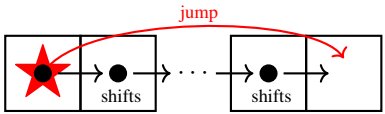
\begin{figure}[]
  \centering
  \begin{tikzpicture}[scale=1]

    \draw[thick] (0,0) rectangle (1,1);
    \draw[thick] (1,0) rectangle (2,1);
    \draw[thick] (3,0) rectangle (4,1);
    \draw[thick] (4,0) rectangle (5,1);

    \draw[->,thick,red]
      (0.5,0.5) .. controls (1,1.3) and (4.1,1.3) .. (4.5,0.7);

    \node[text=red,scale=2.5] at (0.5,0.5) {$\bigstar$};
    \fill (0.5,0.5) circle (0.12);

    \fill (1.5,0.5) circle (0.12);
    \fill (3.5,0.5) circle (0.12);

    \node at (2.525,0.5) {$\cdots$};

    \draw[->,thick] (0.7,0.5) -- (1.3,0.5);

    \draw[->,thick] (1.7,0.5) -- (2.2,0.5);
    \draw[->,thick] (2.8,0.5) -- (3.3,0.5);
    \draw[->,thick] (3.7,0.5) -- (4.3,0.5);

    \node[text=red,anchor=west]   at (2.2,1.3) {\scriptsize jump};
    \node[text=black,anchor=west] at (1.15,.2) {\scriptsize shifts};
    \node[text=black,anchor=west] at (3.15,.2) {\scriptsize shifts};

  \end{tikzpicture}

  \caption{Illustration of the queue-jump step corresponding to block~10 in Figure~\ref{flowchart}. The deployment point is marked in red, and the agents are shown in black. The black arrows represent the individual agent shifts towards $x^*$, while the red arrow depicts the effective result: the agent at the deployment point effectively leaping over the subsequent agent(s).}
  \label{fig:shift}
\end{figure}

We now provide several key remarks regarding the decentralized execution of Algorithm~\ref{algo:dadence} DADENCE. Algorithm \ref{algo:dadence} details the code running on each decentralized agent. Agents exchange observations and, because they follow identical greedy rules, can compute the actions of other agents to reach a consensus on action execution (Line \ref{wh:851}). In Section~\ref{sec:implDAD}, we present a more efficient communication-based implementation where agents compute their individual actions then share it rather than computing every agent's action.


For readability and better understanding, we will reference both the code and the flowchart in Figure~\ref{flowchart}, which gives a high-level representation of the algorithm.

The input to DADENCE is a deployment point $x_d$, an empty set of mobile agents $A$, and $\beta$, which represents the discrete field of view of the agents (see {\it line 1}). On {\it line 3} the function continues until no border is observed. This means that the world is fully covered, so block 2 in Figure~\ref{flowchart} outputs ``False'' and the algorithm terminates. In the exceptional case where $\mathcal F = W$ with $A = \emptyset$, the world is already covered by the deployment point $x_d$ alone. In the case where a border exists, meaning the world is not covered, block 2 outputs ``True'' and the loop proceeds to select the $x^*$ target in line 6 (block 3).

The ``True'' condition of the \texttt{if} statement at {\it line~\ref{835}} (block 4) is triggered when the deployment point is already occupied. This prevents a new agent from spawning. The algorithm then enters the \texttt{while} loop at {\it line~\ref{wh:837}} (block 10), where the agent currently at the deployment point plans to move along a shortest path from $x_d$ to the target $x^*$. Whenever this agent tries to step into a cell that is already occupied, the agent in that cell is also moved one step closer to $x^*$ along a shortest path. This process repeats until a moved agent enters a free cell. In other words, all agents lying on that shortest path form a queue, and each of them shifts forward by exactly one cell toward $x^*$, as illustrated in Figure~\ref{fig:shift}. Conceptually, the agent at the deployment point can be seen as ``jumping'' over this queue of neighboring agents by triggering this chain of shifts, allowing it to advance toward the selected target $x^*$. This queue-like shifting mechanism reduces unnecessary travel, preserves the relative ordering of agents, and prevents collisions.

On the other hand, the \texttt{else} branch at {\it line~\ref{1054}} is entered when no agent occupies the deployment point $x_d$. In this case, all the agents declare their actions (block 5) to move towards their goal $x^*$, as shown in \texttt{for} loop starting at {\it line~\ref{for:844}}. Next, the \texttt{while}-loop (block 6) starting at {\it line~\ref{wh:851}} in the algorithm, verifies that the intended action of an agent does not induce (i) a disconnected visibility graph, or (ii) coverage loss while making sure there is no collision. If an intended action is to result in any of these conflictual scenarios, according to DADENCE, each agent affected changes their declared action to instead stay at their current position until the conflicts are resolved (see Section \ref{algo_implementation} for more implementation details). Non-reverted actions are referred to as {\it valid actions}.

Following the \texttt{while}-loop ending at \emph{line~\ref{858}} (Block 6), $A'$ represents the new, conflict-free positions of the agents. Then, the \texttt{if} statement at \emph{line~\ref{1071}} verifies whether these declared actions result in any displacement (block 7). If at least one agent moves (i.e., $A \neq A'$), the positions are updated to the new positions, as shown in \emph{line~\ref{862}} (corresponds to block 8). Otherwise, a new agent is deployed pursuant to \emph{line~\ref{860}} (block 11). In either case, the algorithm returns to the start of the loop and starts again.

Next, we present formal proofs that the algorithm preserves visibility graph connectivity (Proposition \ref{D_prop_a}) and guarantees the full coverage of the world $W$ (Proposition \ref{D_prop_b}); together, these properties ensure the algorithm halts and returns the desired output.

\begin{proposition}
\label{D_prop_a}
Under Assumptions \ref{asp1}--\ref{asp3}, any modification made to the location of the agents, entailed by DADENCE, maintains the connectivity of the discrete visibility graph $\tilde G(A\cup\{x_d\})$.
\end{proposition}

\begin{proof} In DADENCE, the positions of agents can only be modified at \textit{lines~\ref{833}, \ref{840}, \ref{860}, and \ref{862}}. We now show that any such modification of $A$ at these lines preserves the connectivity of $\tilde G(A\cup\{x_d\})$.

After the agent spawning entailed by either  \textit{lines~\ref{833}} or \textit{lines~\ref{860}}, we have $\tilde G(A\cup\{x_d\}) = \tilde G(A)$. The connectivity follows from the fact that $\tilde G(A\cup\{x_d\})$ was connected before the execution of either of these lines. This holds by induction, with the base case given by the first execution of \textit{line~\ref{833}}, after which we have $\tilde G(A \cup \{x_d\}) = \tilde G(\{x_d\})$, which is connected.

By \textit{line~\ref{840}}, DADENCE attempts to move the agent located at the deployment point $x_d$ along a shortest path from $x_d$ to the target $x^*$. This can lead to many agents shifting positions towards $x^*$ to prevent collisions (as seen in Figure~\ref{fig:shift}). The result of this shifting is as if only the agent at the deployment point jumps over many agents to an unoccupied position $x_k \in W\setminus {A\cup \{x_d \}}$. To note, before the agent is moved, the agents in the visibility graph are $A \cup\{x_d \}$, and the visibility graph is connected. After the agent at the deployment point moves, the set of agents becomes $A \cup\{x_d \} \cup \{x_k \}$. $A \cup\{x_d \} \cup \{x_k \}$ must also form a connected visibility graph, as $x_k$ is placed at a position adjacent to a current agent or the deployment point, meaning it must be connected to at least one agent in the previous visibility graph.

After the execution of \textit{line~\ref{862}} the connectivity is maintained because $\tilde G(A'\cup\{x_d\})$ is connected. This holds because the \texttt{while} loop in {\it lines~\ref{wh:851}} to \ref{858} reverts any action that may destroy the connectivity. 
\end{proof}

\begin{proposition}
\label{D_prop_b}
Under Assumptions \ref{asp1}--\ref{asp3}, the execution of DADENCE guarantees that, eventually, the world $W$ is covered. That is, $A$ becomes a set such that $\tilde F(A \cup \{x_d \}) = W$.
\end{proposition}

\begin{proof}
    We first show that changes to $A$ entailed by DADENCE do not reduce coverage (any subset of the field of view always remains so after any action). Then, we show that the changes to $A$ expand coverage until the collective field of view $\tilde F(A \cup \{x_d \}) = W$. 

    {\it Proof that the coverage is maintained:} The set $A$ changes after the execution of \textit{lines~\ref{833}, \ref{840}, \ref{860}, and \ref{862}}. We prove that after the execution of either of these lines, there does not exist any part of $W$ whose coverage is lost.

    In \textit{lines~\ref{833}} and \ref{860} no coverage is lost as the coverage of the world includes the discrete field of view of the deployment point $x_d$. Therefore, adding $x_d$ to $A$ does not change $\tilde F(A \cup \{x_d \})$.

    The execution of \textit{line~\ref{840}} makes an agent move from the deployment point $x_d$ to another position. Since the deployment point always adds to the discrete field of view, removing an agent from $x_d$ cannot reduce the collective field of view.
  
    Similar to connectivity (see the proof of Proposition \ref{D_prop_a}), after the execution of \textit{line~\ref{862}}, the coverage is maintained because the \texttt{while} loop in {\it lines~\ref{wh:851}} to \ref{858} reverts any action that may reduce coverage.

    {\it Proof that the coverage expands:} The discrete field of view $\mathcal F$ is initialized with the field of view of $x_d$ ({\it line~\ref{832}}). Then, in  {\it line~\ref{835}}, if the condition holds, an agent moves from $x_d$ to a position $x_k$ (see the proof of Proposition \ref{D_prop_a}) along its shortest path to the target position $x^*$. Alternatively, if the condition in  {\it line~\ref{832}} does not hold, all the agents move towards the target following the shortest path respectively, until no more valid actions are possible. Then, a new agent is spawned in, and advances towards the target in the next execution of the \texttt{while} loop defined in  {\it lines~\ref{833,37}--\ref{867}}. The repetition of such moves implies that after a finite number of steps, an agent reaches the target $x^*$, and observes beyond this point, {\it i.e.,} the coverage expands. 
   
\end{proof}

\section{Algorithm implementation and optimization}
\label{algo_implementation}

We now present implementation refinements that improve the performance of both algorithms by shortening the execution time and reducing the number of agents. Time-reduction strategies for CADENCE and DADENCE are given in Sections~\ref{sec:implCAD} and \ref{sec:implDAD}, respectively, and an agent deallocation policy is described in Section~\ref{sec:dealloc}. These refinements do not affect any of the Propositions~\ref{cadence_prop}--\ref{D_prop_b}.

\subsection{Implementation and optimization of CADENCE}
\label{sec:implCAD}

We first note that the high-level formulation of CADENCE in Section \ref{sec:cadence} is not directly implementable: on {\it line~\ref{cadln:5}} agents are effectively assumed to appear instantaneously at their target, whereas in practice they must move along a path from the deployment point. To simplify path planning and remain consistent with DADENCE, we therefore restrict CADENCE to operate in a discrete world $W \subset \mathbb Z^2$ following Assumption~\ref{asp3}, as shown in the ``pixelized'' illustration in Figure \ref{fig:cad_dad_sample}. In this setting, $T_{\mbox{\scriptsize max}}$ is the number of maximum time steps, and at each step every agent may either remain in place or move to a neighboring cell; simultaneous moves by multiple agents still count as a single time step.

\begin{figure}[h!]
  \centering
  \begin{subfigure}[b]{0.225\textwidth}
    \centering
    \includegraphics[
      width=\linewidth,
      trim=0pt 0pt 220pt 0pt,  
      clip
    ]{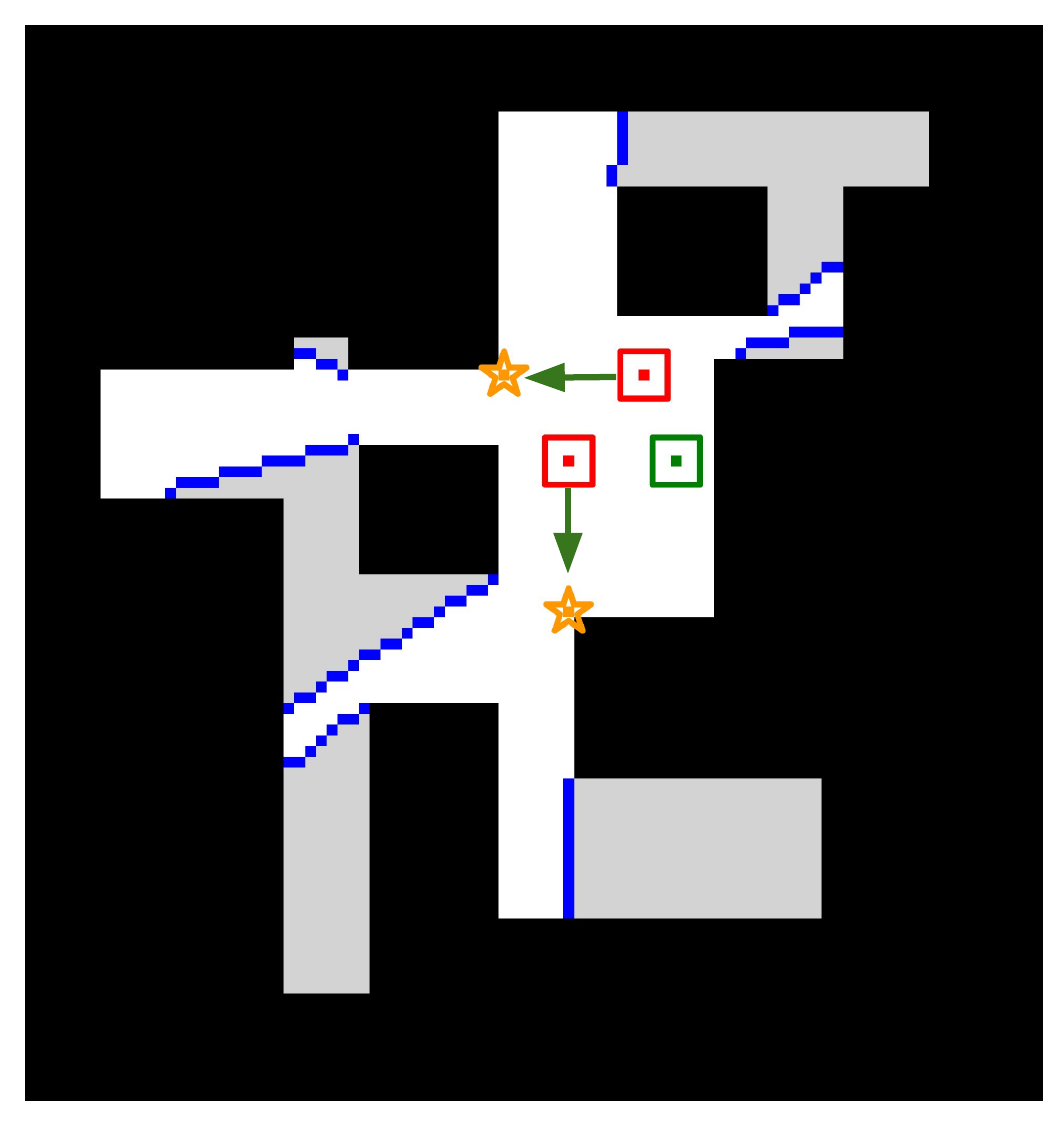}
    \caption{CADENCE}
    \label{fig:cadence}
  \end{subfigure}%
  \hspace{0.5cm}%
  \begin{subfigure}[b]{0.225\textwidth}
    \centering
    \includegraphics[
      width=\linewidth,
      trim=0pt 0pt 220pt 0pt,
      clip
    ]{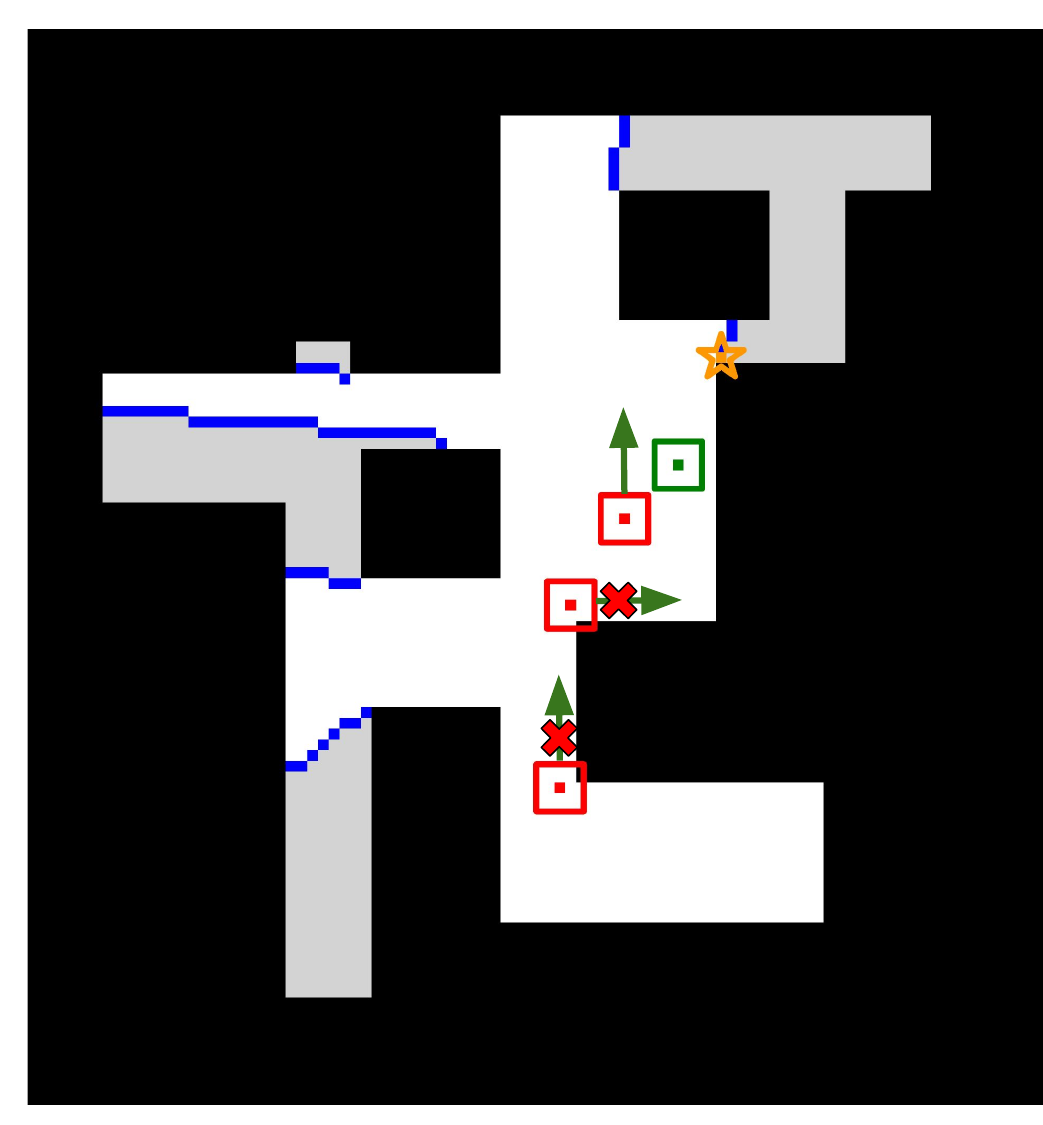}
    \caption{DADENCE}
    \label{fig:dadence}
  \end{subfigure}

  \caption{%
    Sample movement of agents in CADENCE (a) and DADENCE (b), where the deployment point is marked by the green square, the agents are the red, and the goals are the gold stars. Green arrows indicate intended actions, and $\times$'s on the arrows indicate actions that are invalid.
  }
  \label{fig:cad_dad_sample}
\end{figure}


To further accelerate CADENCE, we allow multiple agents to move simultaneously and do not require a new agent to be spawned only after the previous one reaches its destination. Instead, CADENCE can spawn an agent provided a valid corner is observed and no agent is already en route to it. To prevent collisions, at most one agent can be spawned per time step under this rule. Figure~\ref{fig:cadence} illustrates this behavior: two red agents move concurrently toward their respective yellow-star target cells. This strategy reduces the total number of steps, but it requires redefining the \texttt{while} loop as follows:

\begin{algorithm}
  \begin{algorithmic}[1]
    \While{%
      \parbox[t]{\dimexpr\linewidth-3.8em\relax}{%
        \(S' \ne A \ne T\) \textbf{ and } \(S' \ne (A \cup \{x_d\}) \ne (T \cup \{x_d\})\)%
      }%
    }\label{cadln:3bis}
      \If{ \(S' \setminus (T \cup \{x_d\}) \ne \emptyset\) } 
        \State Pick \(x \in S' \setminus (T \cup \{x_d\})\)
        \State Set \(T = T \cup \{x\}\)
        \State \(A = A \cup \{x_d\}\)
      \EndIf
      \State $A = A'$  
      \State $S' := \alpha\big( \,(A \cap T) \cup \{x_d\}\,\big)$ 
      \label{cadln:8bis}
    \EndWhile
  \end{algorithmic}
\end{algorithm}

In the modified \texttt{while} loop, $T$ is the set of target valid corners, $A$ the current agent positions, and $A'$ the agent positions after executing their actions. Each agent independently computes its next move toward its target, so $A'$ is obtained in a distributed manner. CADENCE terminates once every observed valid corner in $T$ has an assigned agent that has reached it.

\subsection{Implementation and optimization of DADENCE}
\label{sec:implDAD}

While DADENCE can be implemented directly as formulated in Section~\ref{sec:dadence}, a naive decentralized approach is computationally expensive because each agent must replicate the action selection logic for the entire group. This can be significantly reduced in two ways. First, similar to how agents share observations, the agents collaboratively compute the target $x^*$ by sharing distances to border cells, which replaces the exhaustive shortest-path computations in Line~\ref{dadln:6}. Furthermore, $x^*$ is recomputed only when it is no longer on the border ($x^* \notin B(\mathcal F)$) (\textit{i.e.}, the vision extends beyond the current border). Second, all the agents declare and share their intended actions, allowing others to verify execution eligibility based on shared data. This eliminates the iterative checks in Lines~\ref{for:844} and~\ref{wh:851}.

We now introduce two optimizations to reduce the number of time steps (actions) of DADENCE. The first optimization is to spawn new agents whenever coverage has not increased for the last $t_h$ steps, rather than waiting until no agent can move. We set $t_h$ to be the average number of steps needed for the current agents to reach the target $x^*$. In large worlds, the distance between $x_d$ and $x^*$ grows as coverage expands, so agents take longer to travel from $x_d$, and each newly revealed border beyond $x^*$ may require additional deployments. By deploying a new agent preemptively once progress has stalled for $t_h$ steps, before $x^*$ is no longer the coverage border, we reduce the total number of time steps.

The second optimization is to modify the underlying navigation graph $\bar{G}(\tilde{F}(A))$, which all agents use to plan paths to $x^*$. In DADENCE, all agents share this graph and greedily follow shortest paths, which can cause repeated reverted actions when they repeatedly choose the same action. To avoid this, we introduce per-agent edge pruning: each agent \(a_i \in A\) maintains its own graph \(\bar{G}_i\!\big(\tilde{F}(A)\big)\)and, whenever a proposed action is reverted, it deletes the corresponding edge from its graph. This forces the agent to switch to alternative paths in subsequent steps. If all outgoing edges at a node are exhausted, the agent stops moving until its local graph is reset. These agent-specific graphs are only updated with all lost edges reinstated when a new target is selected or a new agent is spawned. If all agents have exhausted their edges, a new agent may also be spawned.

To note, this edge-pruning mechanism does not prevent DADENCE from ultimately covering the environment, since lost edges are eventually reinstated. However, if a target point has not yet been reached, pruning can temporarily halt movement. In such cases, a new agent is spawned, allowing DADENCE to resume. Figure~\ref{fig:dadence} illustrates this process: three agents attempt to move toward the same target cell (yellow star), but two intended moves are reverted because they would cause loss of coverage or connectivity. The corresponding edges are crossed out and temporarily removed from the local graphs of the agents, while the action of the third action is valid and can be executed. This prevents the two agents from repeatedly attempting the same invalid actions.

\subsection{Deallocation}
\label{sec:dealloc}
   
Having described strategies to reduce execution time and time steps, we now turn to reducing the number of agents by removing unnecessary {\it terminal-position agents}. A terminal-position agent is one that no longer moves because it has reached its final intended position. An unnecessary agent is defined as follows.

\begin{defi}[Unnecessary terminal-position agent]
    A terminal-position agent at the position $x \in A$ is deemed unnecessary if:
    \begin{enumerate}
        \item Removing it does not reduce coverage of the world. That is, if $F(A\cup \{x_d\}) = F(\,\big(A \setminus \{x\})\cup \{x_d\}\,\big)$
        \item Removing the agent does not disconnect the graph representing the network of remaining agents. That is, if $\tilde G\big(\,(A \cup \{x_d\} \setminus \{x\}\,\big) $ is connected. 
\end{enumerate}
\end{defi}

Algorithm \ref{alg:deallocation-while-for} presents the deallocation process.

\begin{algorithm}
  \caption{Greedy Deallocation}
  \label{alg:deallocation-while-for}
  \begin{algorithmic}[1]
    \State \textbf{Input:} terminal position agents $A$, deployment point $x_d$
    \State Set $\text{changed} \gets \text{True}$
    \While{$\text{changed}$}
        \State $\text{changed} \gets \text{False}$
        \For{each agent $a \in A$}
            \State $A' = A \setminus \{a\}$
            \If{$G(A' \cup \{x_d\})$ is connected 
                      \\  \hspace{18mm} 
                 \textbf{and} $\tilde{F}(A' \cup x_d) = \tilde{F}(A\cup x_d)$}
                \State $A = A'$
                \State $\text{changed} \gets \text{True}$
                \State \textbf{break}
            \EndIf
        \EndFor
    \EndWhile
    \State \textbf{Output:} $A$
  \end{algorithmic}
\end{algorithm}

By design, the algorithm does not result in a reduction of coverage or loss of connectivity. Hence, the deallocation process does not break Propositions \ref{cadence_prop}--\ref{D_prop_b}.

Because CADENCE and DADENCE have different dynamics, their terminal agents are defined differently. In CADENCE, an agent becomes terminal once it reaches its target valid corner, so the deallocation procedure can be invoked at the end of each iteration of the \texttt{while} loop (specifically, after {\it line}~\ref{cadln:6} in Algorithm~\ref{alg:CADENCE}). In DADENCE, agents continue to move until the world is fully covered, so no agent is terminal before completion; deallocation can only be applied once coverage is achieved, {\it i.e.}, after {\it line}~\ref{867}.

Figure~\ref{fig:three_algos} gives an illustration example of the deallocation process. The initial configuration, shown in Figure~\ref{fig:de_alloc_start}, consists of four terminal position agents (purple, pink, blue, orange) deployed to achieve complete coverage of the world, while maintaining connectivity. We then apply the deallocation algorithm, and the result is shown in Figure~\ref{fig:de_alloc_good}. Because the blue agent contributes neither unique coverage nor connectivity, it is unnecessary and can be safely deallocated. In contrast, Figure~\ref{fig:de_alloc_bad} shows that deallocating the orange agent from Figure~\ref{fig:de_alloc_start} breaks both coverage and connectivity, so this agent is necessary and cannot be deallocated.
\begin{figure}[h!]
    \centering
    \begin{subfigure}[b]{0.15\textwidth}
        \centering
        \includegraphics[width=\linewidth]{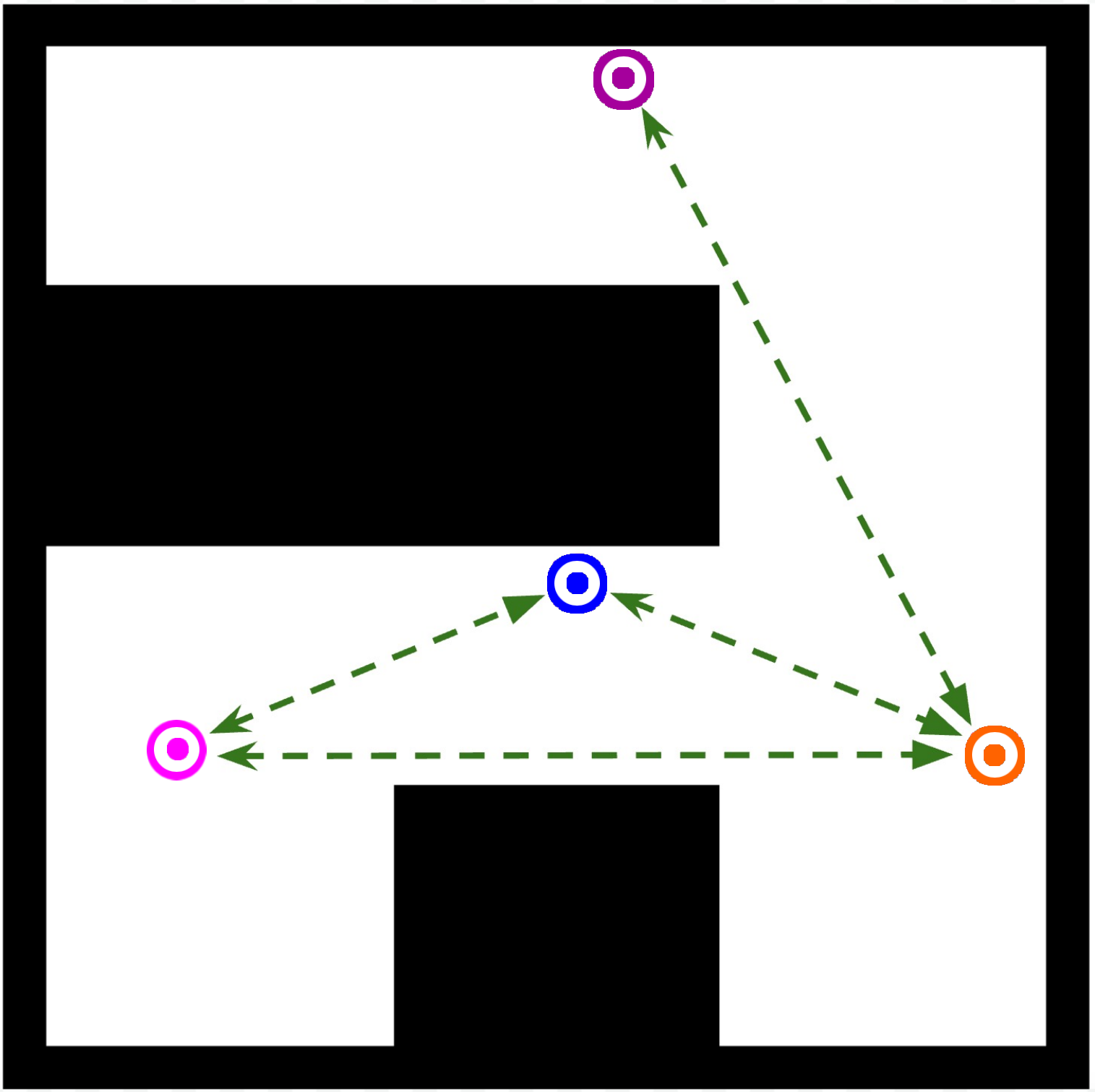}
        \caption{}
        \label{fig:de_alloc_start}
    \end{subfigure}
    \hspace{0.01cm}
    \begin{subfigure}[b]{0.15\textwidth}
        \centering
        \includegraphics[width=\linewidth]{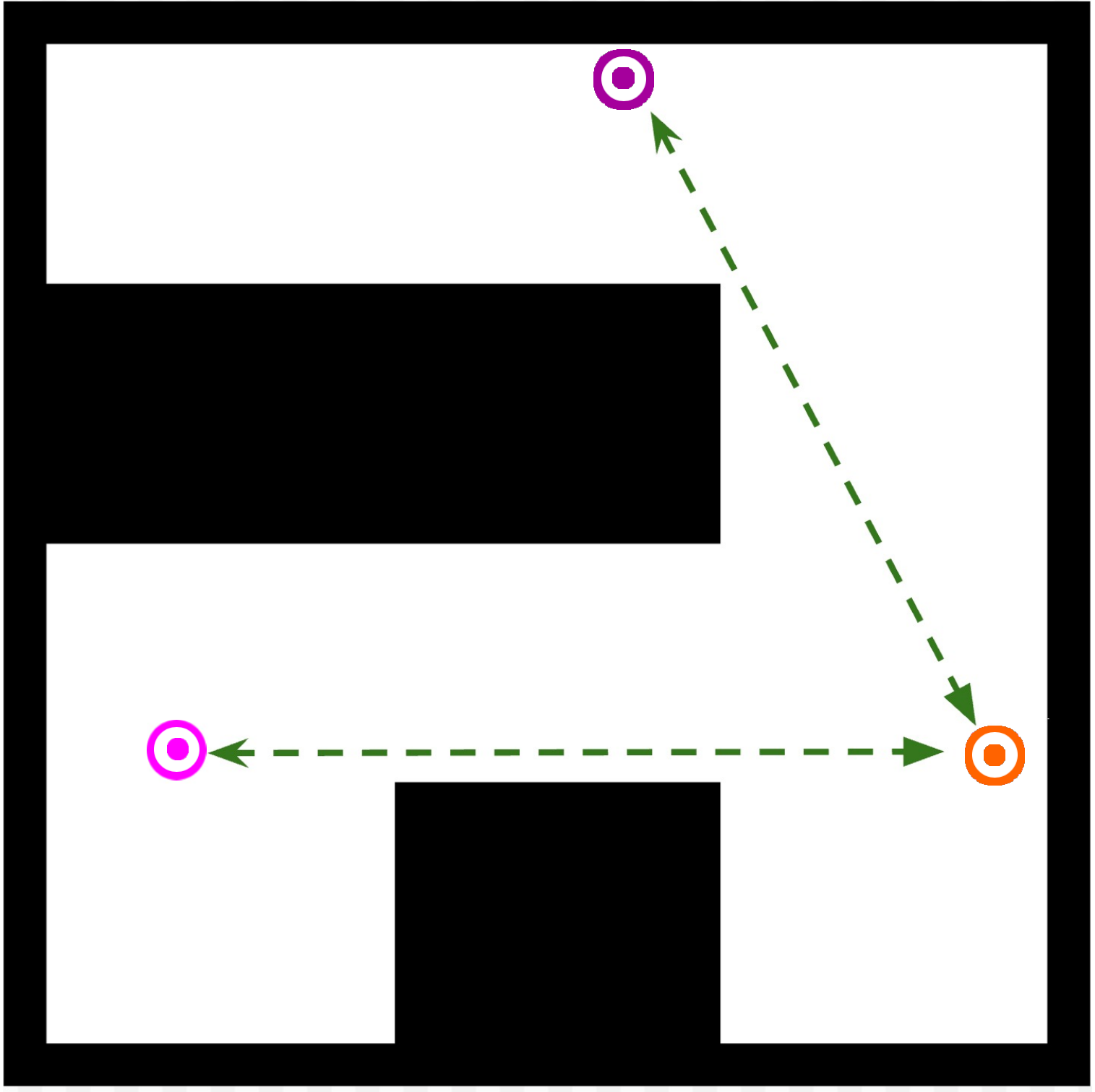}
        \caption{}
        \label{fig:de_alloc_good}
    \end{subfigure}
    \hspace{0.01cm}
    \begin{subfigure}[b]{0.15\textwidth}
        \centering
        \includegraphics[width=\linewidth]{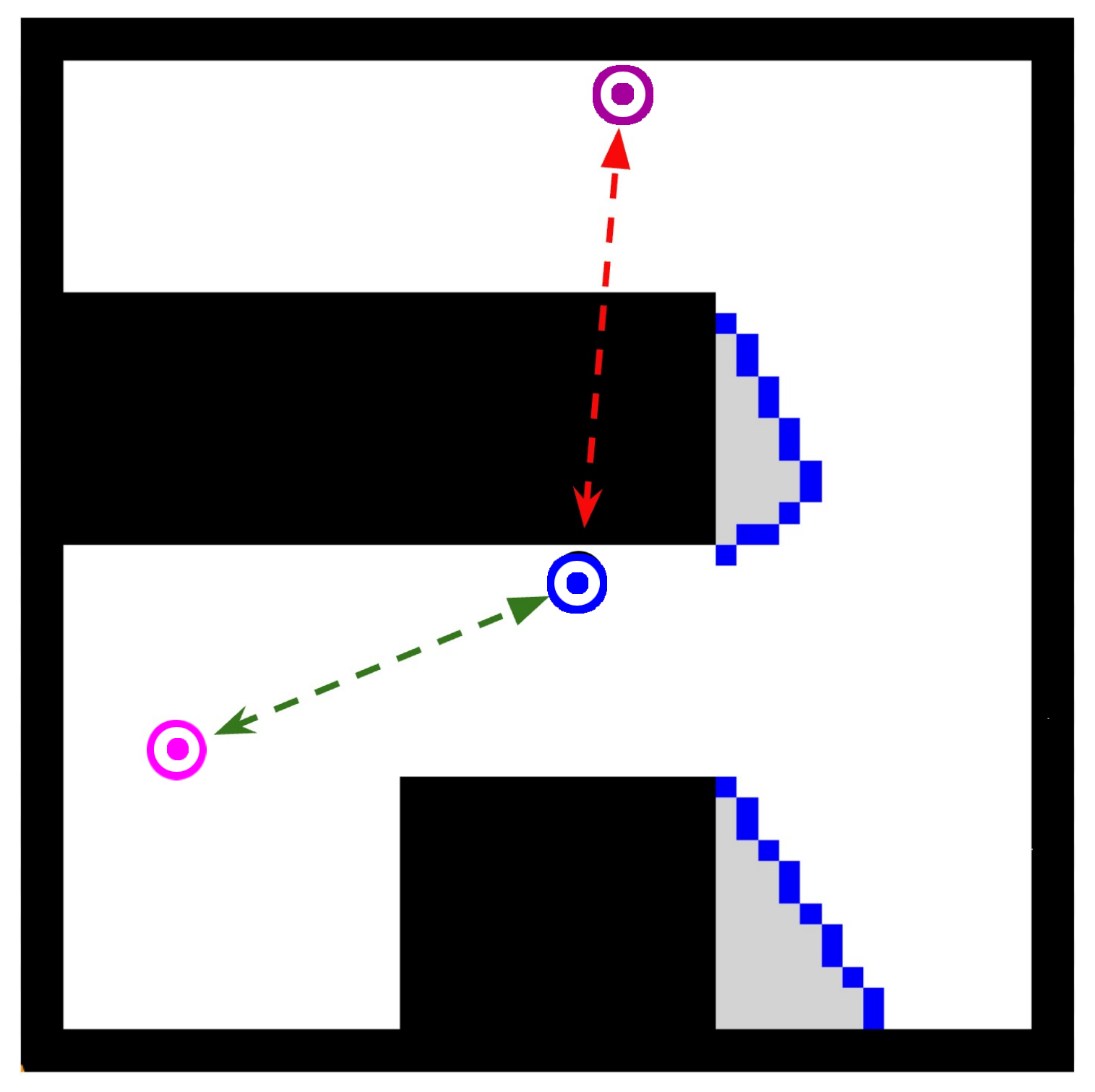}
        \caption{}
        \label{fig:de_alloc_bad}
    \end{subfigure}
    \caption{Example of deallocation. In Figure \ref{fig:de_alloc_start}, four agents (pink, blue, purple, and orange) are deployed to fully cover the world, given the obstacles in black. Figure \ref{fig:de_alloc_good} illustrates the result of deallocating the unnecessary agent (blue), while Figure \ref{fig:de_alloc_bad} depicts the loss of coverage and connectivity when a necessary agent is deallocated.}
    \label{fig:three_algos}
\end{figure}

\section{Comparison through extensive numerical tests}
\label{sec:tests}

To evaluate the ability of CADENCE and DADENCE to deploy MANETs in the POCGAGP, intensive numerical simulations were run on randomly generated discrete worlds of varying complexity (size, number of holes, and corners). We also compared the proposed algorithms to several adapted baseline algorithms for MANET deployment in POCGAGP. In this section, we first present the experimental setup used to measure and compare performance, then introduce the benchmark algorithms. Performance comparison results are reported as a function of world complexity: sixty complexity levels, five distinct worlds per level, and five random deployment points per world, for a total of eight algorithms. That is a total of 12,000 numerical simulations. Simulations for different algorithms were run in parallel across multiple computers, each equipped with an AMD 9950x with 64GB of memory and an RTX 4090.

\subsection{Benchmark environments}

\label{benchmark_env}

We construct test environments that consist of randomly generated dungeons, enabling experiments across a wide range of worlds with varying sizes and complexities, ensuring comprehensive evaluation.

\begin{figure}[h!]
    \centering
    \includegraphics[
        width=\linewidth,
        trim=15.5cm 9.5cm 3cm 5cm, 
        clip
    ]{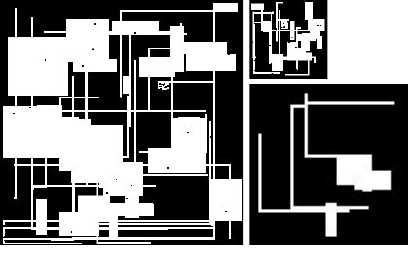}
    \caption{Dungeon examples with sizes $250\times250$ (large), $100\times100$ (medium), and $50\times50$ (small). The largest world is also the most structurally complex, with tunnels and obstacles; the medium world is structurally simplest; and the smallest has intermediate complexity. These examples show that world size alone is not a reliable measure of complexity.}
    \label{fig:snail}
\end{figure}

\subsubsection{Dungeon environments}

A \textit{dungeon environment} is a simulated grid world, commonly used in Robotics, Gaming, and Artificial intelligence. In two dimensions, it represents a confined area containing walls, corridors, obstacles, and rooms. These environments closely match the 2D POCGAGP discrete-world setting, making dungeon simulations an ideal testbed.

All these dungeon worlds will satisfy Assumption~\ref{asp3}. They are drawn from three discrete sizes, measured in cells as defined on page~\pageref{cell}: $50 \times 50$, $100 \times 100$, and $250 \times 250$. Each size class contains worlds with varying numbers of tunnels, rooms, and obstacles (see Figure~\ref{fig:snail}). As the dungeon size increases, the environments generally become more complex to cover, and thus typically require more agents and more time steps to achieve full coverage. This relationship between world size and complexity is, unfortunately, inaccurate and is illustrated in Figure~\ref{fig:snail}, where the largest dungeon is the most complex, but the smallest dungeon is more complex than the medium-sized dungeon. Similarly, properties such as surface area, obstacle count, or corner count are not sufficient on their own to capture this notion of complexity. In each case, even as one increases the world size, total surface area, or number of corners in an attempt to make the environment more complex, it is always possible to construct a world that can be a solution to the POCGAGP with just a single agent.

\begin{figure}[h!]
    \centering
    \includegraphics[width=0.45\textwidth,]{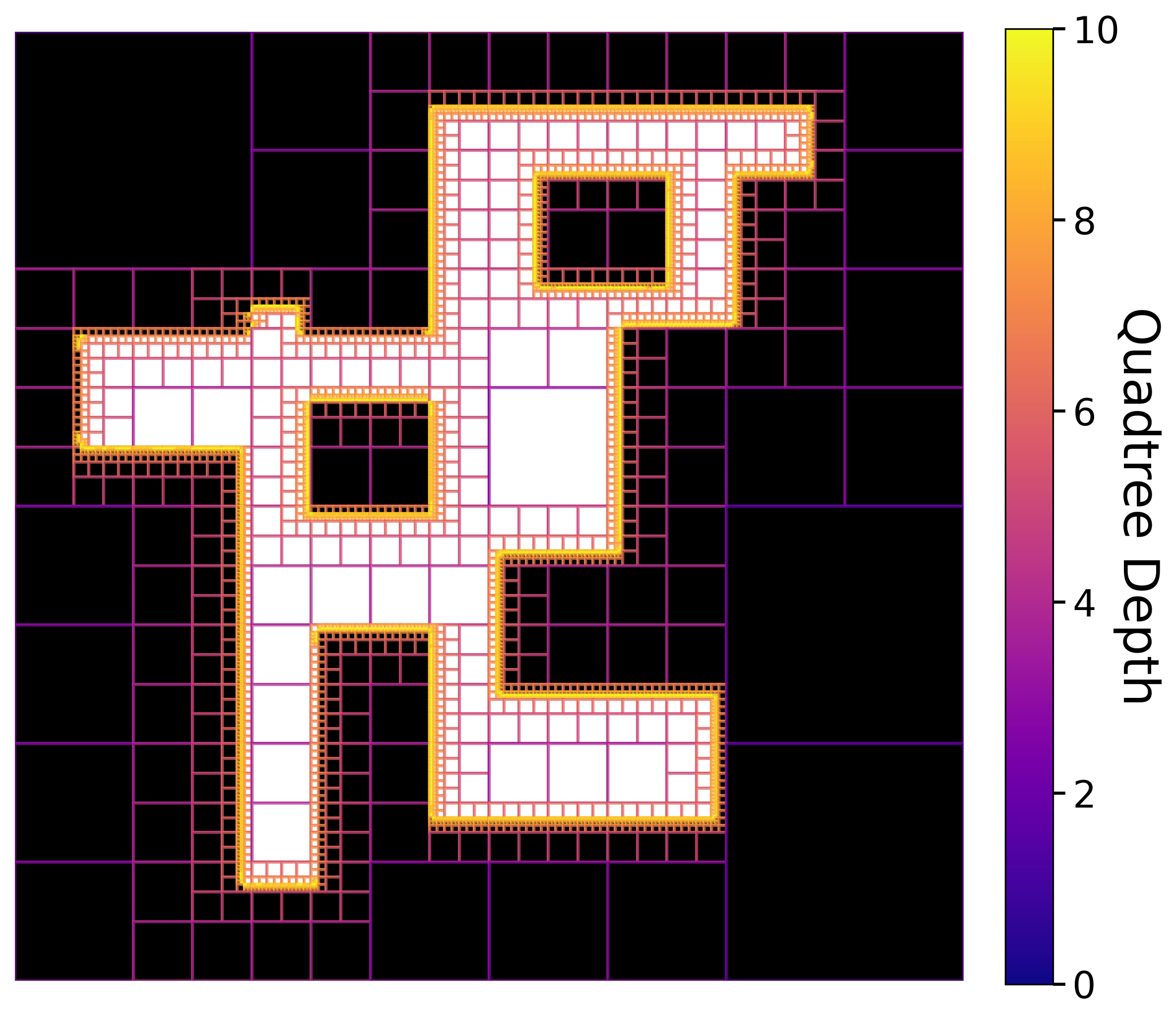} 
    \caption{Quadtree decomposition of the world map seen in Figure~\ref{fig:deployed} for hierarchical representation. Obstacles and outer walls are filled in black, and the free space is in white. The coloured outlines trace each square’s subdivision depth (0 at the coarsest level in deep purple, up to 10 at the finest level in bright yellow).}
    \label{fig:dun_quad}
\end{figure}
We propose to measure coverage complexity by the number of nodes in a quadtree's representation of each dungeon. Each square dungeon is recursively subdivided into smaller regions using a quadtree structure, which continues splitting until the entire world is fully partitioned into hierarchical nodes. An example is shown in Figure~\ref{fig:dun_quad}. Regions with higher obstacle density or fragmentation produce deeper quadtree subdivisions, resulting in brighter colors and a larger total number of nodes. This representation provides a natural proxy for spatial complexity, as it captures both the depth and intricacy of the environment. 

\begin{figure}[h!]
    \centering
    \includegraphics[width=0.5\textwidth,]{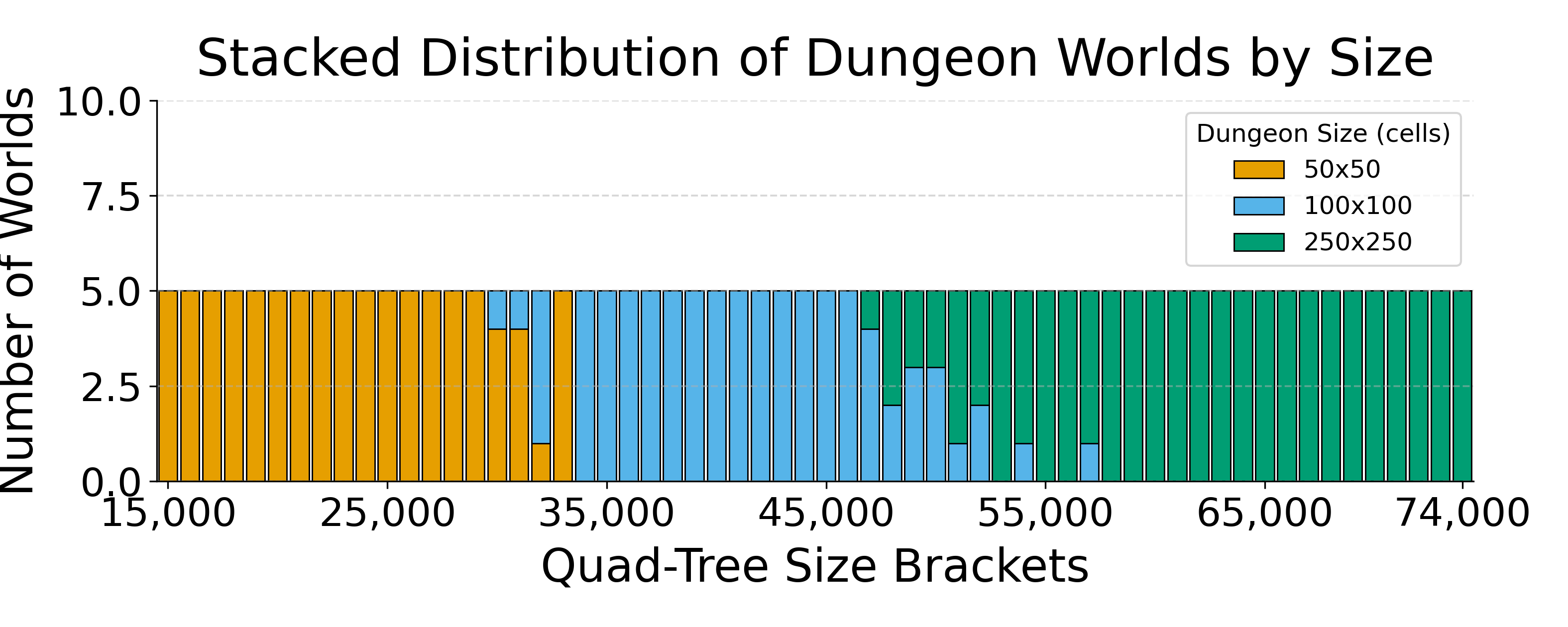} 
    \caption{Dungeon World Histograms by Size}
    \label{fig:histogram}
\end{figure}

Then, to facilitate comparison across dungeons, we introduce the notion of quadtree rank, which groups dungeons based on their total node count, see the histogram in Figure \ref{fig:histogram}. Each rank spans a range of 1{,}000 nodes:  0-999 for rank~0, 1{,}000-1{,}999 for rank~1, and so on.  Each node represents a square division in Figure \ref{fig:dun_quad}.  We generate a total of 15{,}000 dungeons (5{,}000 per dungeon size), sort them by quadtree node count, and randomly select 300 dungeons, five from each rank between 15{,}000 and 74{,}000. The distribution of selected worlds is shown in Figure~\ref{fig:histogram}, color-coded by dungeon size. For instance, one can see that there is a 100x100 cell world having a complexity level of 55,000 nodes, which is more complex than 20 larger $250\times250$ worlds. This illustrates that world size alone does not dictate complexity. 

\begin{remark}
Note that in a continuous world $W \subset \mathbb R^2$ the quadtree is bottomless, {\it i.e.,} an infinite partition is possible, thereby leading to infinite complexity. This further motivates the discretization of $W$.
\end{remark}

\subsubsection{Resource allocation in the tests}

\label{test_info}

A key element of \textsc{POCGAGP} is the use of two resource budgets, $N_{\max}$ and $T_{\max}$, which constrain the algorithms. For each world, we set $N_{\max}$ equal to the bound from Equation~\eqref{cgagp}, guaranteeing sufficient agents in the worst case to cover the space. While this choice is unrealistic in real-world settings, obfuscated by the lack of prior knowledge of the world and resource constraints, it is appropriate for our controlled tests and enables us to see which algorithms can solve \textsc{POCGAGP} with at most this set upper bound. For $T_{\max}$, we scale the step budget with the geometric size of the world: although computational complexity is driven by the number of quadtree nodes, the physical steps required to traverse the environment depend on metric distance, so a size-aware $T_{\max}$ better reflects execution-time constraints. In practice, $50 \times 50$ worlds are allocated $5{,}000$ steps, $100 \times 100$ worlds $10{,}000$ steps, and $250 \times 250$ worlds $30{,}000$ steps.

\subsubsection{Randomness consideration}

Results for each algorithm can vary with different starting positions. To account for this, we run each world five times using five fixed random seeds. Each seed determines the deployment point, and the randomness in the algorithm, if applicable, and because the same seeds are used for all algorithms, they are evaluated from the same five starting positions.

\subsection{Benchmark algorithms}

To evaluate the performance of our proposed algorithms for the POCGAGP, we compare them against adapted MANET deployment algorithms with extensive simulation tests. We first briefly review the benchmark algorithms and then present the results.


\subsubsection{Lattice-forming algorithms}

Lattice structures offer a centralized systematic approach to creating a MANET by deploying agents in regular, grid-like patterns that provide predictable communication coverage~\cite{ccabuk2021max}. These lattices can take the form of many different types of shapes. Two of the most common, the square and triangular layouts, shown in Figure~\ref{fig:lattice}, are used as comparison baselines in this paper.

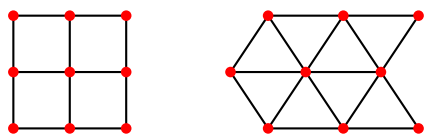
\begin{figure}[h!]
    \centering
    \begin{tikzpicture}

        \draw[thick,black] (0,0) -- (.75,0);
        \draw[thick,black] (1.5,0) -- (.75,0);
        
        \draw[thick,black] (0,.75) -- (.75,.75);
        \draw[thick,black] (1.5,.75) -- (.75,.75);

        \draw[thick,black] (0,1.5) -- (.75,1.5);
        \draw[thick,black] (1.5,1.5) -- (.75,1.5);

        \draw[thick,black] (0,0) -- (0,.75);
        \draw[thick,black] (0,.75) -- (0,1.5);

        \draw[thick,black] (.75,0) -- (.75,.75);
        \draw[thick,black] (.75,.75) -- (.75,1.5);

        \draw[thick,black] (1.5,0) -- (1.5,.75);
        \draw[thick,black] (1.5,.75) -- (1.5,1.5);
    
        \fill[red] (0,0) circle (2pt);
        \fill[red] (.75,0) circle (2pt);
        \fill[red] (1.5,0) circle (2pt);

        \fill[red] (0,.75) circle (2pt);
        \fill[red] (.75,.75) circle (2pt);
        \fill[red] (1.5,.75) circle (2pt);
        
        \fill[red] (0,1.5) circle (2pt);
        \fill[red] (.75,1.5) circle (2pt);
        \fill[red] (1.5,1.5) circle (2pt);

    \end{tikzpicture}
    \hspace{1cm} 
    \begin{tikzpicture}
        \draw[thick,black] (1,0) -- (2,0);
        \draw[thick,black] (1,1.5) -- (2,1.5);
        \draw[thick,black] (1,0) -- (0.5,.75);
        \draw[thick,black] (1,1.5) -- (0.5,.75);
        \draw[thick,black] (2,1.5) -- (2.5,.75);
        \draw[thick,black] (2,0) -- (2.5,.75);
        
        \draw[thick,black] (1,1.5) -- (1.5,.75);
        \draw[thick,black] (2,1.5) -- (1.5,.75);
        \draw[thick,black] (2,0) -- (1.5,.75);
        \draw[thick,black] (1,0) -- (1.5,.75);
        \draw[thick,black] (.5,.75) -- (1.5,.75);
        \draw[thick,black] (2.5,.75) -- (1.5,.75);

        \draw[thick,black] (3,1.5) -- (2,1.5);
        \draw[thick,black] (3,0) -- (2,0);

        \draw[thick,black] (3,1.5) -- (2.5,.75);
        \draw[thick,black] (3,0) -- (2.5,.75);
        
        \fill[red] (1,0) circle (2pt);
        \fill[red] (2,0) circle (2pt);

        \fill[red] (0.5,.75) circle (2pt);
        \fill[red] (1.5,.75) circle (2pt);
        \fill[red] (2.5,.75) circle (2pt);
        
        \fill[red] (1,1.5) circle (2pt);
        \fill[red] (2,1.5) circle (2pt);

        \fill[red] (3,1.5) circle (2pt);
        \fill[red] (3,0) circle (2pt);

    \end{tikzpicture}
    \caption{Sample square lattice and sample triangle lattice}
    \label{fig:lattice}
\end{figure}

In the square lattice model, guaranteed coverage is achieved when the distance between adjacent lattice nodes, $r$, is set to one discretized cell. Given this step size and a sufficient number of agents, specifically, one per cell, the entire environment can be covered in finite time. However, when agents are too few or spaced too far apart, coverage gaps occur, as shown in Figure~\ref{fig:lattice_coverage}(a). Full coverage can be achieved by deploying more agents and minimizing the step size, as shown in Figure~\ref{fig:lattice_coverage}(b).

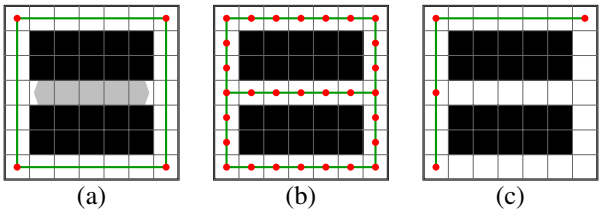
\begin{figure}[h!]
    \centering
    \begin{tikzpicture}[scale = 0.33]
        \fill[white] (-1,-1) rectangle (6,6);
        \fill[lightgray] (0,2) rectangle (5,3);
        \fill[white,thick] (0,2) -- (0,3) -- (.35,3) -- cycle;
        \fill[white,thick] (0,2) -- (0,3) -- (.35,2) -- cycle;
        
        \fill[white,thick] (5,2) -- (5,3) -- (4.65,3) -- cycle;
        \fill[white,thick] (5,2) -- (5,3) -- (4.65,2) -- cycle;
        \draw[thick] (-1.,-1.) rectangle (6.,6.);

    
        \fill[black] (0,0) rectangle (5,2);
    
        \fill[black] (0,3) rectangle (5,5);

        \draw[step=1cm, gray, very thin] (-1,-1) grid (6,6);

        \draw[thick,mydarkgreen] (-0.5,5.5) -- (5.5,5.5);
        \draw[thick,mydarkgreen] (-0.5,5.5) -- (-.5,-.5);
        \draw[thick,mydarkgreen] (5.5,-.5) -- (-.5,-.5);
        \draw[thick,mydarkgreen] (5.5,5.5) -- (5.5,-.5);

        \fill[red] (-0.5,-.5) circle (4pt);
        \fill[red] (-0.5,5.5) circle (4pt);
        \fill[red] (5.5,5.5) circle (4pt);
        \fill[red] (5.5,-.5) circle (4pt);
        
        \node at (2.5, -1.7) {(a)};
    
    \end{tikzpicture}
    \label{fig:sparse}
    \hspace{0.2cm} 
    \begin{tikzpicture}[scale = 0.33]
        \fill[white] (-1,-1) rectangle (6,6);

        \draw[thick] (-1.,-1.) rectangle (6.,6.);

        \fill[black] (0,0) rectangle (5,2);
    
        \fill[black] (0,3) rectangle (5,5);

        \draw[step=1cm, gray, very thin] (-1,-1) grid (6,6);

        \draw[thick,mydarkgreen] (-0.5,2.5) -- (5.5,2.5);
        \draw[thick,mydarkgreen] (-0.5,5.5) -- (5.5,5.5);
        \draw[thick,mydarkgreen] (-0.5,5.5) -- (-.5,-.5);
        \draw[thick,mydarkgreen] (5.5,-.5) -- (-.5,-.5);
        \draw[thick,mydarkgreen] (5.5,5.5) -- (5.5,-.5);

        \fill[red] (-0.5,-.5) circle (4pt);
        \fill[red] (-0.5,5.5) circle (4pt);
        \fill[red] (5.5,5.5) circle (4pt);
        \fill[red] (5.5,-.5) circle (4pt);

        \fill[red] (-0.5,4.5) circle (4pt);
        \fill[red] (-0.5,3.5) circle (4pt);
        \fill[red] (-0.5,2.5) circle (4pt);
        \fill[red] (-0.5,1.5) circle (4pt);
        \fill[red] (-0.5,0.5) circle (4pt);
        
        \fill[red] (5.5,4.5) circle (4pt);
        \fill[red] (5.5,3.5) circle (4pt);
        \fill[red] (5.5,2.5) circle (4pt);
        \fill[red] (5.5,1.5) circle (4pt);
        \fill[red] (5.5,0.5) circle (4pt);

        \fill[red] (4.5,-.5) circle (4pt);
        \fill[red] (3.5,-.5) circle (4pt);
        \fill[red] (2.5,-.5) circle (4pt);
        \fill[red] (1.5,-.5) circle (4pt);
        \fill[red] (0.5,-.5) circle (4pt);
        
        \fill[red] (4.5,5.5) circle (4pt);
        \fill[red] (3.5,5.5) circle (4pt);
        \fill[red] (2.5,5.5) circle (4pt);
        \fill[red] (1.5,5.5) circle (4pt);
        \fill[red] (0.5,5.5) circle (4pt);
        
        \fill[red] (4.5,2.5) circle (4pt);
        \fill[red] (3.5,2.5) circle (4pt);
        \fill[red] (2.5,2.5) circle (4pt);
        \fill[red] (1.5,2.5) circle (4pt);
        \fill[red] (0.5,2.5) circle (4pt);
        \node at (2.5, -1.7) {(b)};
    \end{tikzpicture}
    \hspace{0.2cm} 
    \begin{tikzpicture}[scale = 0.33]
        \fill[white] (-1,-1) rectangle (6,6);

        \draw[thick] (-1.,-1.) rectangle (6.,6.);

        \fill[black] (0,0) rectangle (5,2);
    
        \fill[black] (0,3) rectangle (5,5);

        \draw[step=1cm, gray, very thin] (-1,-1) grid (6,6);

        \draw[thick,mydarkgreen] (-0.5,5.5) -- (5.5,5.5);
        \draw[thick,mydarkgreen] (-0.5,5.5) -- (-.5,-.5);

        \fill[red] (-0.5,-.5) circle (4pt);
        \fill[red] (-0.5,5.5) circle (4pt);
        \fill[red] (5.5,5.5) circle (4pt);

        \fill[red] (-0.5,2.5) circle (4pt);
        
        \node at (2.5, -1.7) {(c)};
    \end{tikzpicture}
    \caption{This figure illustrates the square lattice algorithm in a discrete environment. Figure (a) shows agents deployed with a step size of 6, while (b) shows a denser deployment with a step size of 1. Figure (c) then shows the final state after deallocation, which results from the deployment in (b).}
    \label{fig:lattice_coverage}
\end{figure}

In contrast, triangular lattice models do not provide the same coverage guarantees in 2D orthogonal worlds. As shown in Figure~\ref{fig:lattice}, agents in triangular lattices are not deployed strictly above or below previous positions. Therefore, deploying agents in a triangular lattice in the same environment often results in incomplete coverage, as depicted in Figure~\ref{fig:triangle_lat}(a)--(c).

\begin{figure}[h!]
    \centering
    \begin{tikzpicture}[scale = 0.33]
        \fill[white] (-1,-1) rectangle (6,6);
        \fill[lightgray] (0,2) rectangle (5,3);
        \fill[lightgray] (5,0) rectangle (6,6);
        \fill[white,thick] (5,0) -- (6,0) -- (6,.2) --  cycle;
        \fill[white,thick] (0,2) -- (0,3) -- (.35,3) -- (0.05,2) -- cycle;
        \fill[lightgray] (0,5) rectangle (5,6);
        \fill[white,thick] (0,5) -- (0,6) -- (.15,6) -- cycle;
        \draw[thick] (-1.,-1.) rectangle (6.,6.);

    
        \fill[black] (0,0) rectangle (5,2);
    
        \fill[black] (0,3) rectangle (5,5);

        \draw[step=1cm, gray, very thin] (-1,-1) grid (6,6);

        \draw[->,red,thick] (-0.5,-0.5) -- (0.5,0.5);
        \draw[->,darkgreen,thick] (-0.5,-0.5) -- (0.5,-0.5);

        \fill[red] (-0.5,-.5) circle (4pt);
        
        \node at (2.5, -1.5) {(a)};
    \end{tikzpicture}
    \hspace{0.2cm}
    \begin{tikzpicture}[scale = 0.33]
        \fill[white] (-1,-1) rectangle (6,6);
        \fill[lightgray] (0,2) rectangle (5,3);
        \fill[lightgray] (5,0) rectangle (6,6);
        \fill[white,thick] (5,0) -- (6,0) -- (6,.2) --  cycle;
        \fill[white,thick] (0,2) -- (0,3) -- (.35,3) -- (0.05,2) -- cycle;

        \fill[lightgray] (0,5) rectangle (5,6);
        \fill[white,thick] (0,5) -- (0,6) -- (.15,6) -- cycle;
        
        \draw[thick] (-1.,-1.) rectangle (6.,6.);
    
        \fill[black] (0,0) rectangle (5,2);
    
        \fill[black] (0,3) rectangle (5,5);

        \draw[step=1cm, gray, very thin] (-1,-1) grid (6,6);

        
        \draw[black] (-0.5,-0.5) -- (0.5,-0.5);
        \draw[->,red,thick] (0.5,-.5) -- (1.5,0.5);
        \draw[->,darkgreen,thick] (0.5,-.5) -- (-0.5,0.5);
        \draw[->,darkgreen,thick] (0.5,-.5) -- (-0.5,-0.5);
        \draw[->,darkgreen,thick] (0.5,-.5) -- (1.5,-0.5);

        \fill[red] (-0.5,-.5) circle (4pt);
        \fill[red] (0.5,-.5) circle (4pt);
        
        \node at (2.5, -1.5) {(b)};
    \end{tikzpicture}
    \hspace{0.2cm}
    \begin{tikzpicture}[scale = 0.33]
        \fill[white] (-1,-1) rectangle (6,6);
        \fill[lightgray] (0,2) rectangle (5,3);
        \fill[white,thick] (0,2) -- (0,3) -- (.45,3) -- (0.05,2) --cycle;
        
        \fill[white,thick] (5,2) -- (5,3) -- (4.55,3) --(4.95,2) -- cycle;

        \fill[lightgray] (0,5) rectangle (5,6);
        \fill[white,thick] (0,5) -- (0,6) -- (.15,6) -- cycle;
        
        \fill[white,thick] (5,5) -- (5,6) -- (4.85,6) -- cycle;
        \draw[thick] (-1.,-1.) rectangle (6.,6.);

        \fill[black] (0,0) rectangle (5,2);
    
        \fill[black] (0,3) rectangle (5,5);

        \draw[step=1cm, gray, very thin] (-1,-1) grid (6,6);

        \draw[->,red,thick] (-0.5,0.5) -- (0.5,1.5);
        \draw[->,red,thick] (5.5,0.5) -- (4.5,1.5);

        \draw[thick,mydarkgreen] (-0.5,-0.5) -- (5.5,-0.5);
        \draw[thick,mydarkgreen] (-0.5,0.5) -- (-.5,-0.5);
        \draw[thick,mydarkgreen] (5.5,0.5) -- (5.5,-0.5);

        \fill[red] (-0.5,-.5) circle (4pt);
        \fill[red] (-0.5,.5) circle (4pt);
        \fill[red] (0.5,-.5) circle (4pt);
        \fill[red] (1.5,-.5) circle (4pt);
        \fill[red] (2.5,-.5) circle (4pt);
        \fill[red] (3.5,-.5) circle (4pt);
        \fill[red] (4.5,-.5) circle (4pt);
        \fill[red] (5.5,-.5) circle (4pt);
        \fill[red] (5.5,.5) circle (4pt);
        
        \node at (2.5, -1.5) {(c)};
    \end{tikzpicture}
    
    \caption{The diagrams show the progression of agent deployments using triangle lattices, from the initial configuration as showed in (a), through the intermediate configuration (b) to the final one, depicted in (c). Red arrows indicate triangle lattice positions that are invalid, while green arrows represent valid next deployment  directions.}
    \label{fig:triangle_lat}
\end{figure}
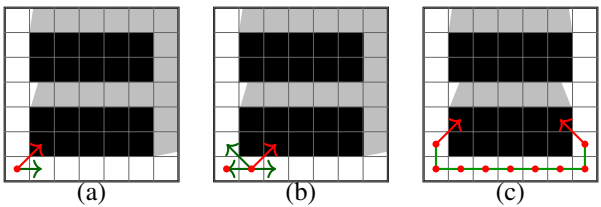

For both lattice types, very small \( r \) values produce overly dense deployments that may exceed \( N_{\max} \) as seen in Figure~\ref {fig:lattice_coverage}(b). The lattice algorithms significantly benefit from deallocation to reduce the number of agents, as seen in Figure~\ref{fig:lattice_coverage}(c).

\subsubsection{Artificial potential fields}

Artificial Potential Field (APF) is a widely used technique for MANETs and robot deployments~\cite{howard2002mobile}. It relies on virtual forces~\cite{khatib1986real} generated by other agents ($\boldsymbol{F}_a$) and obstacles ($\boldsymbol{F}_o$), with the total force $\boldsymbol{F} = \boldsymbol{F}_o + \boldsymbol{F}_a$. These forces can move agents into static equilibrium points, where the net force is zero, analogous to Lagrange points~\cite{broucke1979traveling} in celestial mechanics. A key advantage is that APF enables decentralized deployment: each agent acts using only relative distances to nearby objects and agents, so the overall sensor network operates as a MANET. Traditionally, APF models assume a continuous world ({\it e.g.}, $\mathbb{R}^2$); we implement APF in a discretized environment for consistent comparison, computing forces from distances between discrete cells.

A well-known limitation of APF algorithms is that they do not guarantee complete coverage or network connectivity, especially in complex environments. Agents may become trapped in local minima (static equilibrium points) and be unable to expand. Once all agents are trapped and stop moving, the algorithm terminates, and the agents are considered to have reached their terminal positions for deallocation.

\subsubsection{Frontier deployment algorithms}

Incremental Self-Deployment Algorithm (ISDA) incrementally deploys a MANET by placing new agents at the frontier, or edge, of the current coverage area~\cite{howard2002incremental}. This approach enables a systematic expansion of the network, as each new agent is placed at a location at the edge of the current coverage, ensuring that every agent remains connected to at least one other agent. Given a sufficient number of agents, the ISDA can guarantee full spatial coverage while maintaining network connectivity. In this paper, we follow a previously used implementation, according to which the next point is chosen at random along the edges of coverage points~\cite{howard2002incremental}. As in previous algorithms, the ISDA is extended with deallocation: once an agent reaches its target position, it enters a terminal position and becomes eligible for deallocation.

\subsubsection{Multi-agent reinforcement learning}

\begin{figure}[h!]
  \centering
  \includegraphics[width=0.5\textwidth]{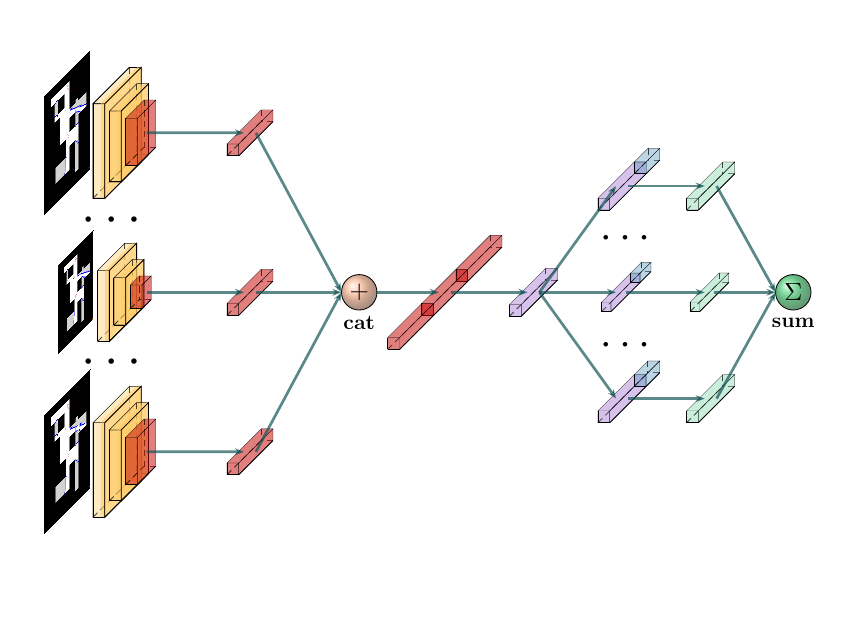}
  \caption{Neural network architecture for VDN}
  \label{fig:nnarch_vdn}
\end{figure}

Multi-agent reinforcement learning (MARL) has emerged as a powerful framework for robotics, where multiple agents learn coordinated policies. Fully cooperative MARL tasks, such as deploying MANETs to cover spaces, can be formulated as a Dec-POMDP~\cite{oliehoek2016concise}, defined by $\mathcal M = \bigl\langle Z,\,A,\,\mathcal{O},\,\mathcal{R},\,\mathcal{P},\,n,\,\gamma \bigr\rangle.$ This tuple consists of $n$ agents, global states $Z$, a shared action space $A$, and an observation space $\mathcal{O}$. Due to partial observability, each agent $i$ receives an observation $o_i^t = O(s;\,i)$ at time $t$ and selects an action $a_i^t \sim \pi_i(\,\cdot\mid o_i^t)$. The resulting joint action is $\boldsymbol{a}^t = \bigl(a_1^t,\dots,a_n^t\bigr) \in A^n$. Then, the system transitions to state $s'$ by dynamics $\mathcal{P} = P\bigl(s' \mid s, \, \boldsymbol{a}^t\bigr)$. In addition, all the agents receive a shared reward $R(s, \,\boldsymbol{a}^t)$. The collective goal is to find a joint policy $\boldsymbol{\pi}=(\pi_1,\dots,\pi_n)$ that maximizes the expected discounted return, $J(\boldsymbol{\pi}) = \mathbb{E}\Bigl[\;\sum_{t=0}^{\infty} \gamma^t\,R\bigl(s^t,\, \boldsymbol{a}^t\bigr)\Bigr],$ where $\gamma$ is the discount factor.

We use two methods as MARL benchmarks for MANETs deployment, each representative of the Centrally Trained Decentralized Executed (CTDE) and the Decentralized Trained Decentralized Executed (DTDE) classes. CTDE methods leverage global information during training to optimize a joint objective~\cite{rashid2020monotonic, sunehag2017value, yu2022surprising}, whereas DTDE methods operate fully independently, with coordination emerging implicitly as agents learn to optimize local objectives that collectively advance the global goal~\cite{jiang2022i2q, tampuu2017multiagent}. During deployment, both paradigms require each agent $i$ to select an action $a_i^t$ based only on its local observation $o_i^t$, without access to the full system state. To improve exploration in complex worlds, we apply Temporally Extended $\epsilon$-greedy exploration~\cite{dabneytemporally}. For both methods, the per-step reward is defined as
\begin{align*}
  \mathcal{R}(s^t, \boldsymbol{a}^t) = & \alpha \cdot \mathrm{AreaCovered} - \beta \cdot \mathrm{CollisionCost} \\ 
  & \quad   - \delta \cdot \mathrm{DisconnectionCost} \cdot \mathrm{NumAgentLost},
\end{align*}
where $\mathrm{AreaCovered}$ is the total area covered by all agents, $\mathrm{CollisionCost}$ penalizes collisions with walls or other agents, and $\mathrm{DisconnectionCost}$ penalizes agents disconnected from the communication graph. $\alpha, \beta, \delta$ are reward scales to control the scale. Collision penalties are assigned to the responsible agent, while coverage and disconnection costs are shared equally among all agents. This reward design encourages agents to maximize coverage while maintaining network connectivity. The DTDE and CTDE methods that we implemented are the following. 

\subsubsection*{DTDE: Independent Q-Learning (IQL)~\cite{tampuu2017multiagent}} This algorithm applies Deep Q‐Learning to a multi‐agent setting by training separate Deep Q‐Networks for each agent under a DTDE paradigm. Each agent \(i\) seeks to approximate its own optimal action‐value function \(Q_i^*\) using only local rewards. The optimality follows the Bellman equation: $Q_i^*(s, a) = \mathbb{E}_{s'}\bigl[r_i + \gamma \max_{a'} Q_i^*(s', a') \mid s, a\bigr].$ Although this approach is simple and computationally less intensive than coordinated methods, it struggles with global coordination because each agent optimizes its objective independently, and the environment becomes nonstationary from any single agent’s viewpoint.

\subsubsection*{CTDE: Value Decomposition Network (VDN)~\cite{sunehag2017value}} This method approximates the joint action-value function as \(Q(\boldsymbol{o},\boldsymbol{a})\approx\sum_{i=1}^nQ_i(o_i,a_i)\), where \(\boldsymbol{o}=(o_1,\dots,o_n)\) and \(\boldsymbol{a}=(a_1,\dots,a_n)\). During training, each agent minimizes the shared loss function $L_{\mathrm{VDN}}=\mathbb{E}_{\boldsymbol{o}', \boldsymbol{a}, \boldsymbol{o}}\bigl[\bigl(R(s,\boldsymbol{a})+\gamma\max_{\boldsymbol{a}'}Q(\boldsymbol{o}',\boldsymbol{a}')-Q(\boldsymbol{o},\boldsymbol{a})\bigr)^2\bigr]. $
By decomposing the joint action-value function into a sum of individual value functions, VDN fosters implicit coordination, as agents learn to maximize a shared global reward while executing their policies independently.

\begin{figure*}
  \centering
  \includegraphics[width=\textwidth]{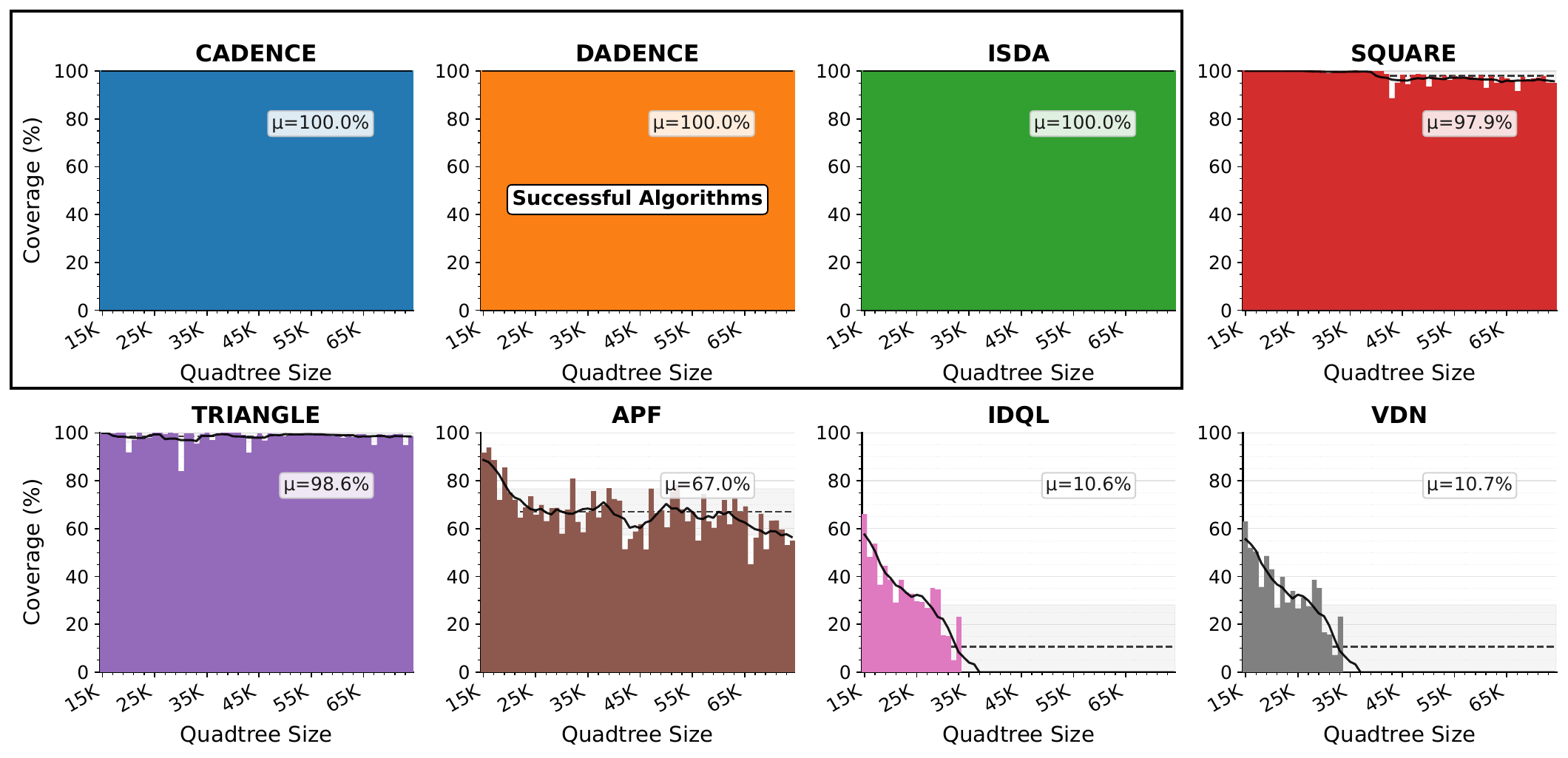}
  \caption{Coverage Data across all experiments and all algorithms. Each plot contains information about the average coverage (represented by $\mu$) and a moving average line showing the average coverage.}
  \label{fig:coverage_data}
\end{figure*}

Both algorithms were trained end-to-end following the setup of \cite{mnih2015human}. At each time step, the input comprises three consecutive grid-world frames as the observation history. These are processed by convolutional layers (kernel size 3, stride 4) with channel sizes \([1,\,32,\,64,\,64]\) to extract spatial features. The resulting feature vector is augmented with a one-hot agent identity and the normalized positions of other agents, then passed through four fully connected layers (hidden size 256) to estimate action values. The only architectural difference is that VDN shares the convolutional backbone among all agents to enhance training efficiency and stability. Figure~\ref{fig:nnarch_vdn} illustrates the VDN structure. Both models were trained using the Adam optimizer~\cite{adam2014method} with a learning rate of \(3 \times 10^{-4} \). We run the deallocation algorithm at the end of the execution.  

Although promising, MARL-based methods face several key limitations. They lack formal coverage guarantees, so some regions in large or complex environments may remain underexplored, leading to poor policy performance. Training is often unstable because each agent’s updates shift the environment distribution for others, creating nonstationarity that can cause oscillations or divergence, compounded by high sample complexity and hyperparameter sensitivity. Scalability is also a challenge. The joint state-action space grows exponentially with the number of agents, increasing memory and computation demands, and complicating coordination in large teams. Moreover, the multi-agent end-to-end training style creates Q-value caveats, leading to imprecise state-action distributions and potentially suboptimal or unstable policies.

These MARL algorithms can only be tested on $50 \times 50$ worlds; with the current network size, larger worlds such as $100 \times 100$ are infeasible due to computational limits. This highlights a key scalability issue in MARL: increasing neural network capacity improves capability, but at a substantial computational cost.

\subsection{Results of numerical tests}
\label{results}

To compare the performance of the algorithms described above, including CADENCE and DADENCE, we describe the numerical tests that we performed, making emphasis on four primary metrics: coverage (in \%), step count, final agent count, and maximal agent count. 

Coverage represents the total area of the environment covered by agents with maintained communication, given as a percentage of the full area covered (0--100\%). Step count denotes the total number of discrete time steps required for algorithm execution, with each deployed agent permitted one action per step. Final agent count indicates the number of agents (including the deployment point pseudo-agent) remaining upon completion of the algorithm after deallocation (Algorithm~\ref{alg:deallocation-while-for}). Finally, the maximum agent count is the peak number of agents (including the pseudo-agent) present simultaneously in the environment during the execution. The step limit of $T_{max}$ is the upper bound on the number of steps an algorithm may take until no more agents can be spawned in, but deallocated agents are allowed to finish their transition to the deployment point even if the step limit is exceeded.

\begin{figure}[h!]
  \centering
  \includegraphics[width=0.47\textwidth]{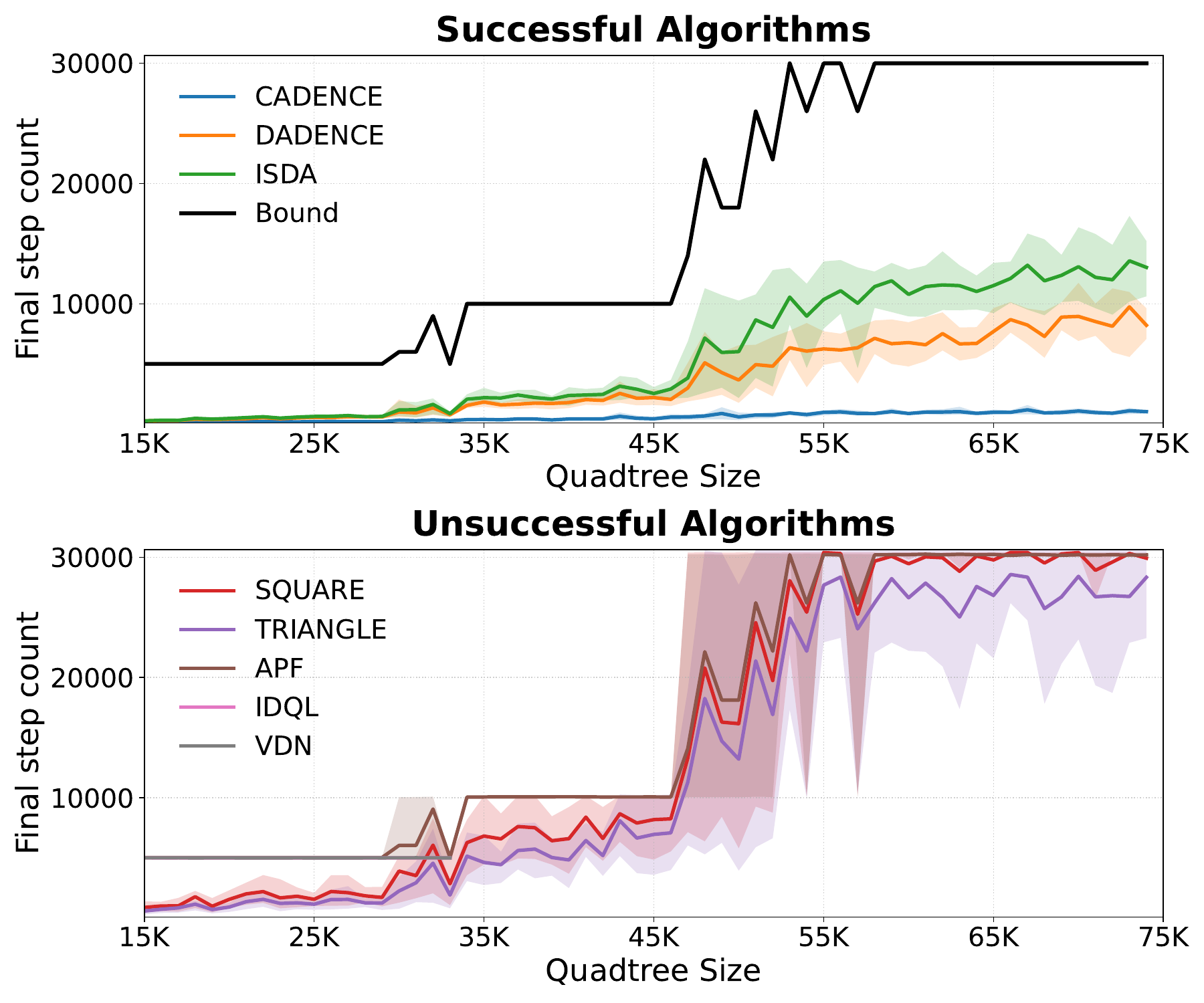}
  \caption{Number of steps taken by algorithm across all experiments. The black line represents the average maximum number of steps allowed at each point. This is an average based on the step limits from each world's size as seen in Section~\ref{test_info}.}
  \label{fig:data_steps}
\end{figure}

\begin{figure}[h!]
  \centering
  \includegraphics[width=0.47\textwidth]{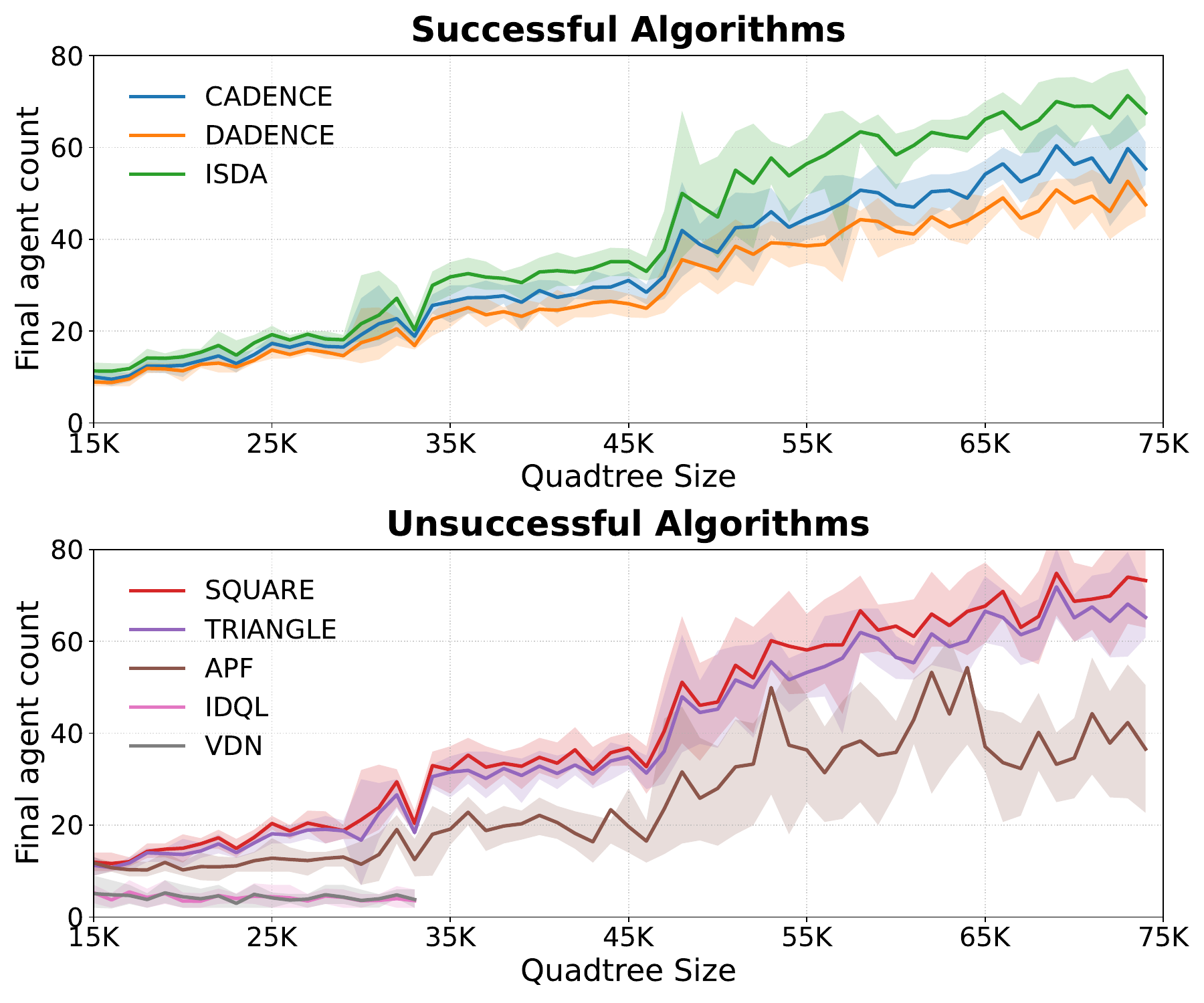}
  \caption{Final number of agents for each algorithm across all experiments. The bound line is omitted from this image as it would affect the scaling of the plot. It is the same bound as seen in Figure~\ref{fig:data_max}.}
  \label{fig:data_final}
\end{figure}

\begin{figure}[h!]
  \centering
  \includegraphics[width=0.47\textwidth]{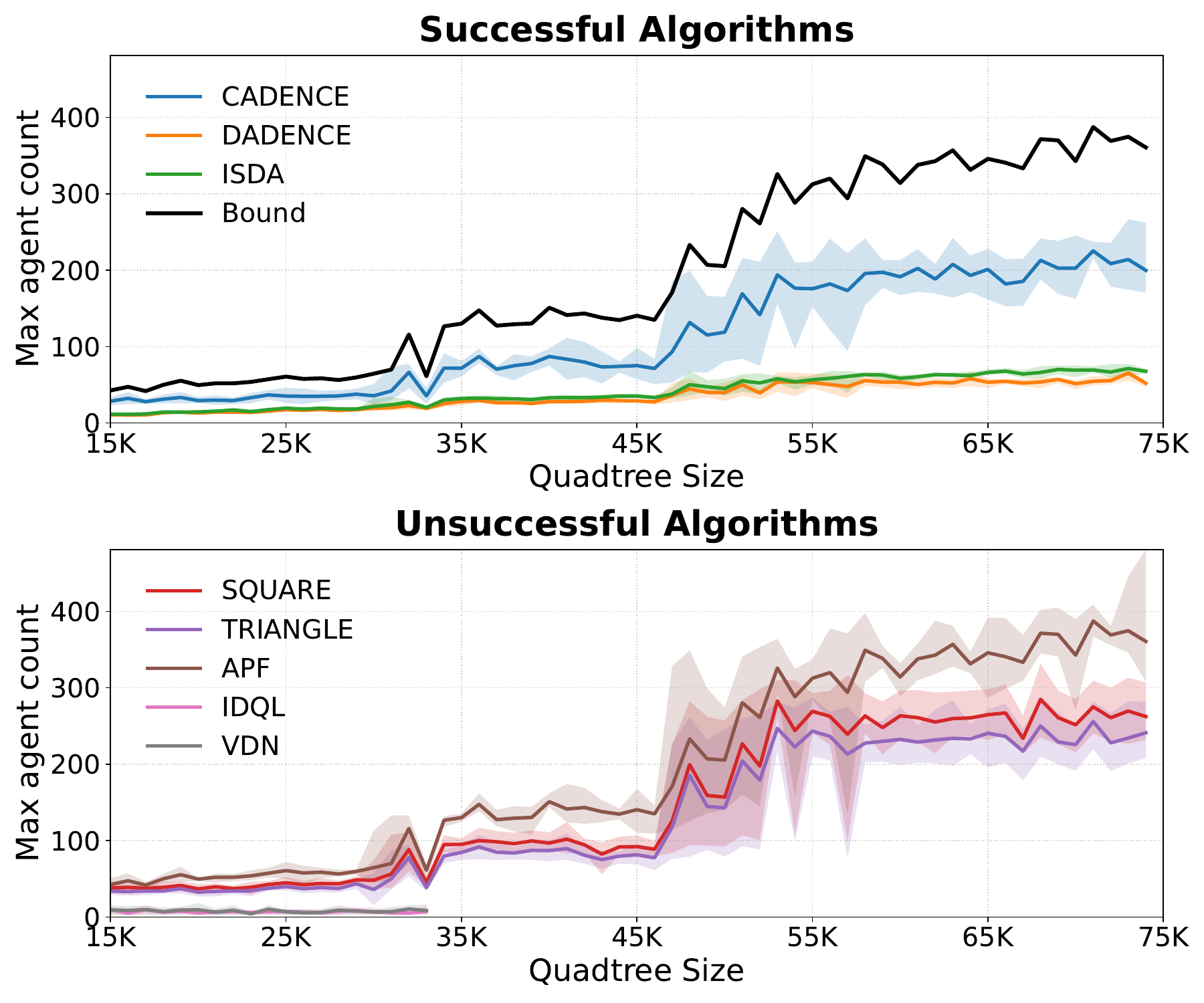}
  \caption{Max number of agents for each algorithm across all experiments. The bound line represents the average maximum number of agents as given from Equation~\ref{457}.}
  \label{fig:data_max}
\end{figure}

\begin{figure}[h!]
  \centering
  \includegraphics[width=0.47\textwidth]{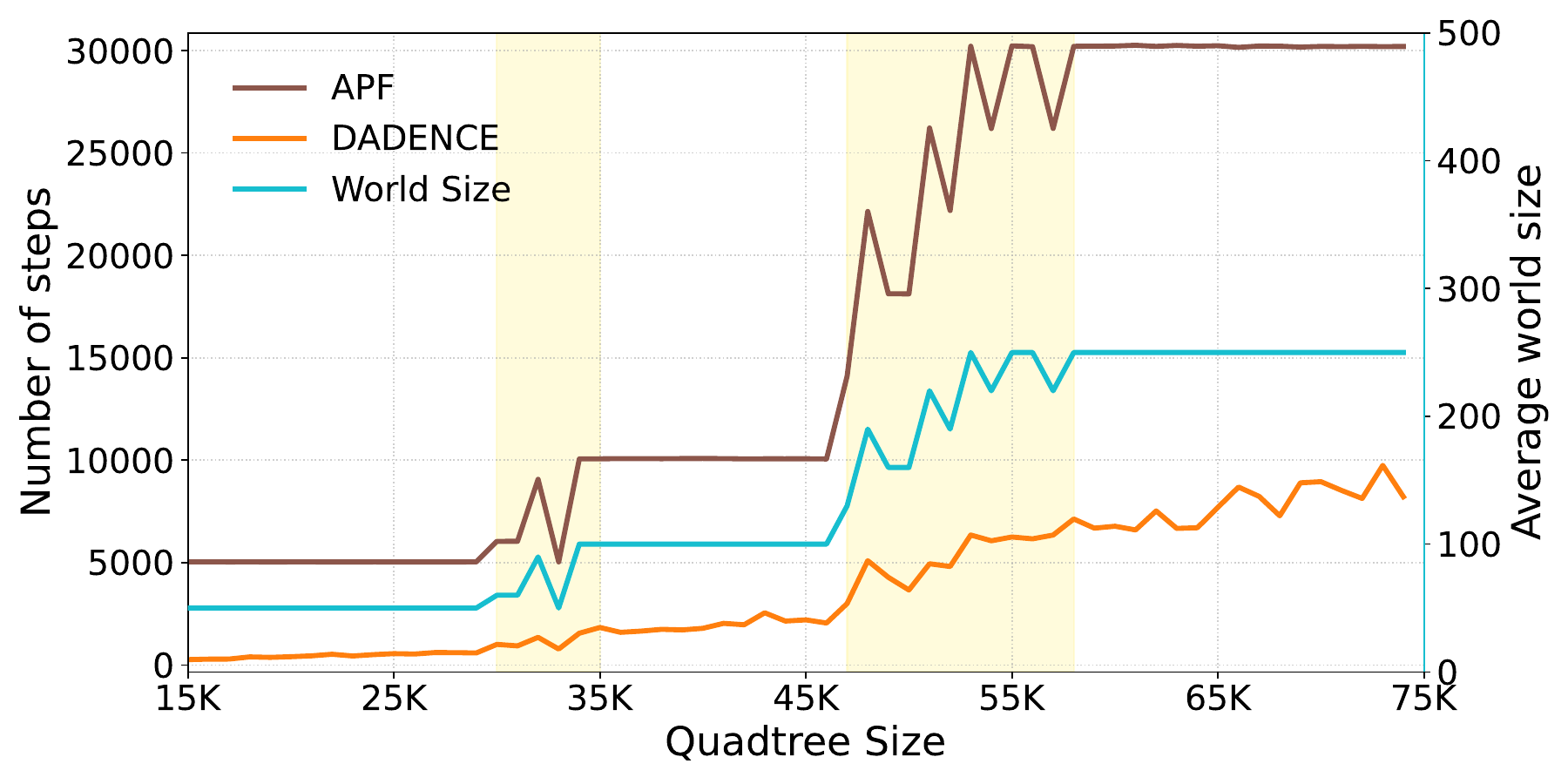}
  \caption{Comparison of the average number of steps required by DADENCE and APF as a function of the average size of the test worlds. Highlighted regions indicate ranges where the world size varies substantially.}
  \label{fig:world_size_steps}
\end{figure}

\begin{figure}[h!]
  \centering
  \includegraphics[width=0.47\textwidth]{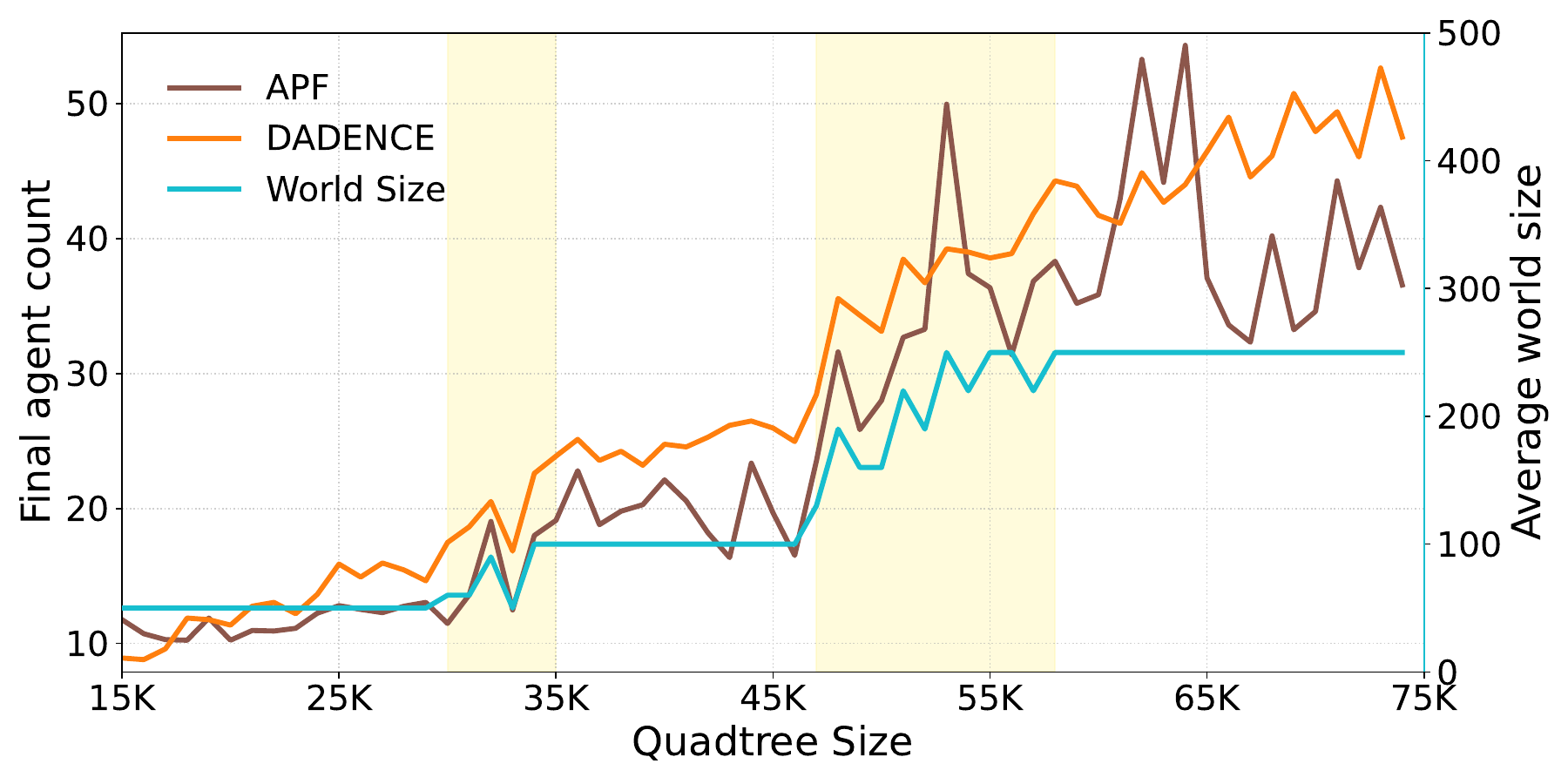}
  \caption{Comparison of the average final agent count for DADENCE and APF as a function of the average size of the test worlds. Highlighted regions indicate ranges where the world size varies substantially.}
  \label{fig:world_size_final}
\end{figure}

\begin{figure}[h!]
  \centering
  \includegraphics[width=0.47\textwidth]{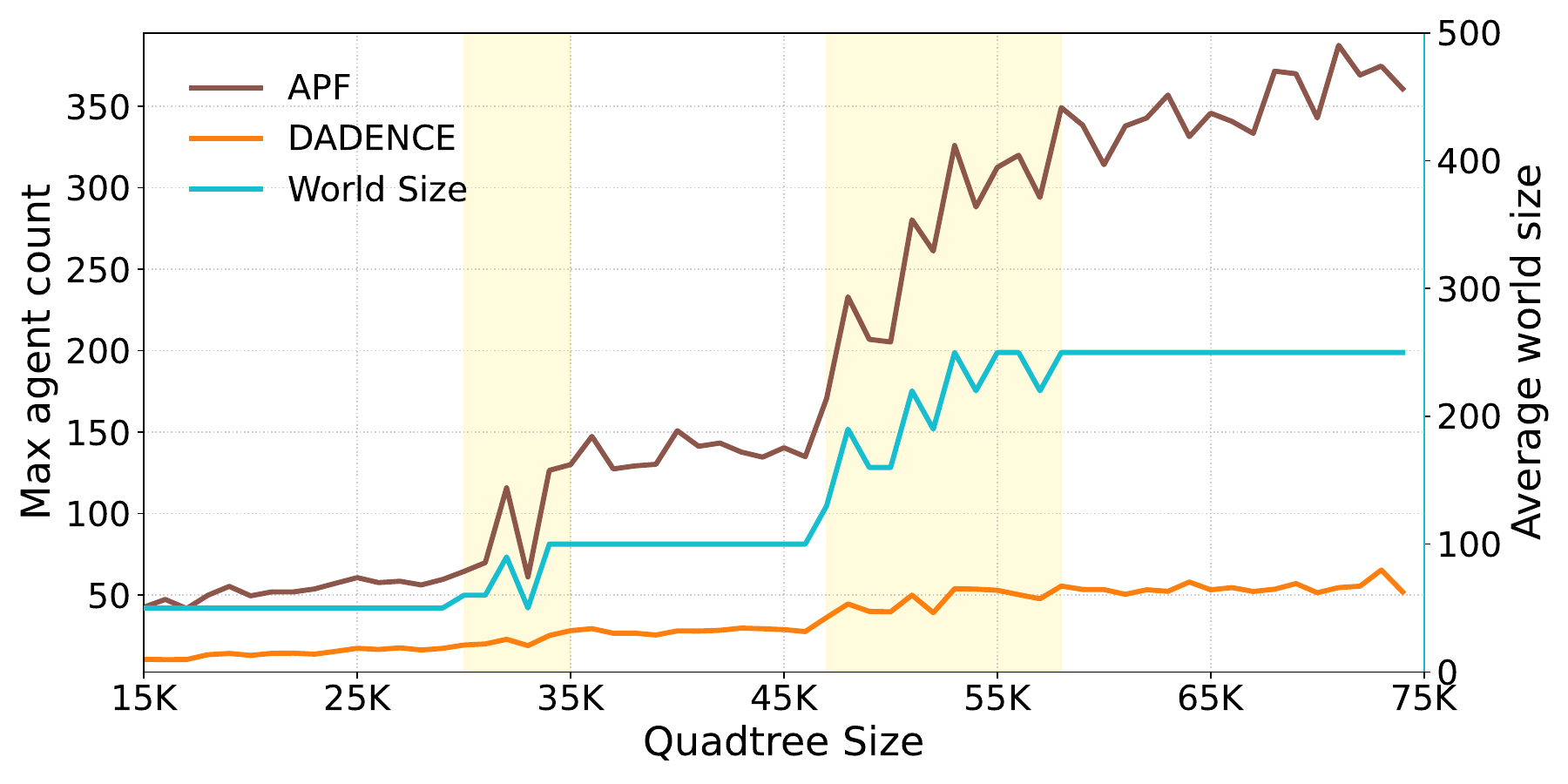}
  \caption{Comparison of the average max agent count for DADENCE and APF as a function of the average size of the test worlds. Highlighted regions indicate ranges where the world size varies substantially.}
  \label{fig:world_size_max}
\end{figure}

\begin{figure}[h!]
  \centering
  \includegraphics[width=0.47\textwidth]{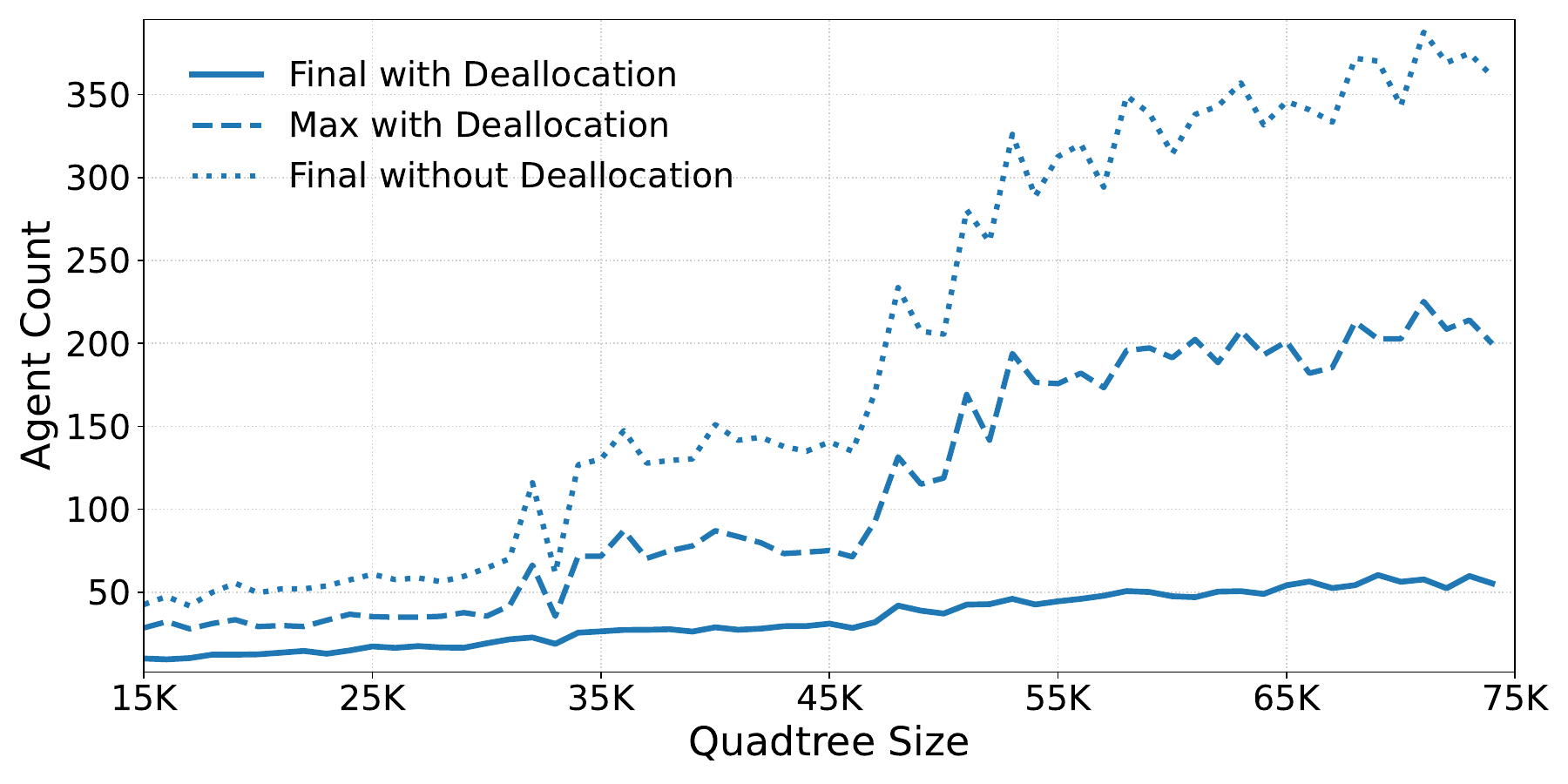}
  \caption{A visual comparison showing the effect of deallocation on CADENCE agent final utilization}
  \label{fig:cadence_de_alloc}
\end{figure}

The results for all tests, grouped by category, are shown in Figures \ref{fig:coverage_data}–\ref{fig:data_max}. Across Figures~\ref{fig:data_steps}, \ref{fig:data_final}, and \ref{fig:data_max}, the results are each split into two subfigures; the top panel shows the successful algorithms and the bottom panel shows the inconsistently successful algorithms, and the figures report step count, final agent count, and maximum agent count, respectively. 

Each aspect of comparison answers specific questions about the ability of each algorithm across the tests. Coverage indicates whether or not the algorithms have been able to solve the POCGAGP on the orthogonal plane with maintained coverage and connectivity, bounded by $N_{\max}$ and $T_{\max}$. Results of coverage below $100\%$ for a given algorithm indicate that the algorithm failed to consistently solve the underlying problem; such an algorithm is then considered {\it unsuccessful}. Then, for successful algorithms, {\it i.e.,} those that consistently achieve complete coverage, we consider how \emph{efficiently} they achieve 100\% coverage. We assess efficiency along two axes: time (total step count) and agent use, including the peak number of agents active at any time (max agent count) and the number remaining at termination (final agent count). The lower these values are, the more resource-usage efficient and algorithm is. Additionally, a comparison can be made between the maximum agent count and the final agent count, showing how many unnecessary agents are deployed at peak but are then no longer needed. Ideally, a smaller difference between these two counts indicates a more efficient deployment with fewer unnecessary agents being utilized. 

Figures~\ref{fig:world_size_steps}--\ref{fig:world_size_max} display the relationship between the number of steps, final number of agents, and max number of agents, and the average world sizes. As shown in Figure~\ref{fig:histogram}, some quadtree–size columns contain a mixture of dungeon sizes rather than a single, consistent world size. The light blue line seen in Figures~\ref{fig:world_size_steps}--\ref{fig:world_size_max} gives an average of world size to be used as a comparison.

Finally, we report the number of agents lost during deployment. Loss of agents arises exclusively in the APF algorithm, which lacks a mechanism to guarantee that deployed agents maintain communication connectivity throughout their movements. In contrast, DADENCE and MARL incorporate mechanisms that prevent agent loss, while CADENCE, ISDA, and the Lattice models deploy agents only into previously explored regions, ensuring that each agent remains connected to the network. Across \(1{,}500\) trials APF deploys nearly \(3\times10^{5}\) agents and loses about \(200\) of them \(\bigl(\approx 6.7\times10^{-4}\text{, i.e., }0.067\%\bigr)\). These lost agents are excluded from APF’s coverage results, shown in Figure~\ref{fig:coverage_data}, and are treated as removed agents in Figure~\ref{fig:data_final}. However, they are still counted in Figure~\ref{fig:data_max}, since the lost agents were initially connected to the network at the start of deployment.

\section{Discussion}
\label{Discussion}

In this section, we analyze our experimental results across all benchmark algorithms and environments. We begin by justifying our choice to measure world complexity using quadtree size.

\subsection{Quadtree complexity measure}
In (Section~\ref{benchmark_env}), we claim that the quadtree size is a more accurate measure of world complexity than raw world size. Figures~\ref{fig:data_steps}--\ref{fig:data_max} support this claim: across all three metrics (step count, final agent count, and maximum concurrent agent count), costs increase with complexity when measured by quadtree size.

Restricting attention to \emph{successful} algorithms in Figures~\ref{fig:data_steps}--\ref{fig:data_max}, we can see that when \emph{world size is fixed}, the metrics still rise with complexity. For instance, for quadtree sizes of $60{,}000$–$74{,}000$ (Figure~\ref{fig:histogram}), all worlds are $250\times250$ cells, yet step count, final agent count, and maximum agent count all trend upward. A similar pattern appears for quadtree sizes of $15{,}000$–$25{,}000$, where all worlds are $50\times50$ cells.

We also observe that fluctuations in \emph{average world size} can affect the \emph{average number of steps}: in Figure~\ref{fig:world_size_steps}, spikes in average world size (yellow) align with step spikes, especially for APF and much less so for DADENCE, revealing algorithm-dependent sensitivity. In contrast, this effect is largely absent in Figures~\ref{fig:world_size_final} and~\ref{fig:world_size_max}, where world-size spikes have little apparent impact on final and maximum agent counts. This reinforces the argument that world size alone is not a reliable proxy for complexity: if it were, its spikes would systematically affect all metrics. Intuitively, as worlds grow, more agents are not necessarily needed to solve POCGAGP, but agents must typically traverse longer paths, increasing the number of steps.

\subsection{General performance comparison}
As we have now empirically analyzed the quadtree-based complexity measure, we begin the analysis of comparing CADENCE and DADENCE to the benchmark algorithms. To do this, we classify algorithms into one of two previously mentioned categories: successful and unsuccessful algorithms. By observing which algorithms consistently covered the worlds in Figure \ref{fig:coverage_data}, it can be seen that among the eight algorithms evaluated, only three, CADENCE, DADENCE, and ISDA, consistently achieve complete coverage across all worlds. In contrast, APF, Square Lattices, Triangle Lattices, IDQL, and VDN fail to consistently cover the world.

Square lattice performance is primarily constrained by time: agents are deployed to every free square until coverage is complete. As the number of agents is capped by \(N_{\max}\); once this cap is reached, the algorithm must wait for agents to be deallocated. In practice, this cap does not consistently affect the algorithm; the maximum agent count seldom hits the limit (Figure~\ref{fig:data_max}), whereas the time budget \(T_{\max}\) is consistently exhausted (Figure~\ref{fig:data_steps}). This indicates that the stricter constraint is simply the time limit. The coverage data result shows the algorithm functions across almost all $50 \times 50$ and $100 \times 100$ and fails across most larger $250 \times 250$ worlds. Triangular lattices suffer from the same underlying issues as square lattices and, in addition, do not guarantee that coverage can be achieved even with infinite time and agents. This lack of guarantee can be seen in the coverage data, where the inability to consistently cover the world is more evenly spread out across all world sizes. APF is primarily limited by the inability of agents to move freely, as they can become stuck in local minima between the various forces pushing on each other. This is the fundamental reason why APF, even in infinite time, has no guarantee of full world coverage~\cite{howard2002mobile}. 

The two MARL algorithms effectively suffer from the same underlying issue, preventing them from covering worlds. The algorithms fail at generalizing agents over uncertainty in spaces. The algorithms in some cases perform so poorly that as the agents are deployed one by one, eventually one does not move from the deployment point, which blocks all other agents from spawning. This is why the MARL algorithms deploy so few agents even at max usage as seen in Figure~\ref{fig:data_max}. In isolated training cases on only one world, with one start location, the MARL algorithms can solve the POCGAGP on the orthogonal plane with maintained coverage and connectivity, but only in cases with no more than five to ten agents. This failure in MARL is consistent with its well-documented scalability issues in complex environments, especially for decentralized algorithms \cite{liu2024scaling,gronauer2022multi,amato2024introduction}.

Taken together, the shortcomings of the failing algorithms are not uniform. MARL performs the worst: high computational cost, poor generalization, no coverage guarantee, and, as can be seen in Figure~\ref{fig:coverage_data}, those two algorithms have the worst coverage results. APF, while simpler, is plagued by its susceptibility to local minima and disconnection. Triangle lattices fare better, as their structured deployment provides more consistent results than APF or MARL. Square Lattices perform at the top of this underperforming group, achieving reliable coverage in smaller environments but scaling poorly as world size increases. This performance hierarchy, MARL and APF at the bottom, Triangle Lattices in the middle, and Square Lattices at the top, is clearly reflected in the coverage results graphs, where the divide between partial and complete coverage widens with increasing complexity.

For these failing algorithms, it is not meaningful to compare step counts, final number of agents, and maximum number of agents, since they do not actually solve the  POCGAGP on the orthogonal plane. For example, IDQL did not fully cover any world, and when it attempts to, it covers very little of each world. Consequently, it utilizes far fewer agents than APF (Figs.~\ref{fig:data_final} and \ref{fig:data_max}). One could compare which algorithm fails more efficiently, by covering more of the world or using fewer resources, but this would solve a different problem than the POCGAGP on the orthogonal plane with maintained coverage and connectivity.

We now turn to the algorithms that consistently solve the POCGAGP problem under the $N_{max}$ and $T_{max}$ bounds. Between CADENCE and DADENCE, an interesting dynamic emerges. While both algorithms significantly underuse the allocated resources, DADENCE consistently utilizes fewer agents, whereas CADENCE achieves coverage in fewer steps. This divergence can be traced to structural differences: DADENCE does not restrict agents to static guard positions, whereas CADENCE enforces vertex-guard placement. Within the CGAGP framework, algorithms restricted to vertex guards (such as CADENCE) are known to admit less favorable worst-case bounds compared to point guard algorithms (such as DADENCE) with a bound of \(\frac{n-2}{2}\)~\cite{zylinski2004cooperative}. Moreover, DADENCE agents remain mobile throughout deployment, reorganizing to maximize coverage as exploration progresses. On the other hand, the reason why CADENCE takes fewer steps is that each agent has their own goal and is simply moving towards it. This policy, according to which all the agents move towards their own targets, prevents any collectively organized re-organization, which DADENCE does benefit from, where all agents share the same goal. This structural distinction manifests in an average difference: CADENCE utilizes an average of \(3{,}100\) fewer steps than DADENCE, while DADENCE uses roughly 4 fewer final agents (after deallocation) and, finally, uses almost an average of 79 fewer maximum agents than CADENCE. Both algorithms remain below the time-step limit and the maximum-agent limit, shown in Figure~\ref {fig:data_max}.

Both algorithms consistently outperform ISDA. Although ISDA does not over-deploy as much as CADENCE (Figure~\ref{fig:data_max}), it requires more steps than both CADENCE and DADENCE (Figure~\ref{fig:data_steps}) and, after deallocation, more final agents remain than CADENCE and DADENCE (Figure~\ref{fig:data_final}); it uses many more agents than DADENCE at peak (Figure~\ref{fig:data_max}). Operationally, CADENCE and ISDA are similar in that they move one agent at a time to expand coverage. The key difference is the agent placement policy: CADENCE selects deployment positions more judiciously. Its superior agent efficiency indicates the advantage of placing agents at valid corners. While edge placements can occasionally beat CADENCE in agent count, such cases are rare and not systematic.

\begin{figure}[t]
    \centering
    \includegraphics[
        width=0.5\textwidth,
        trim=10 0 180 0,
        clip
    ]{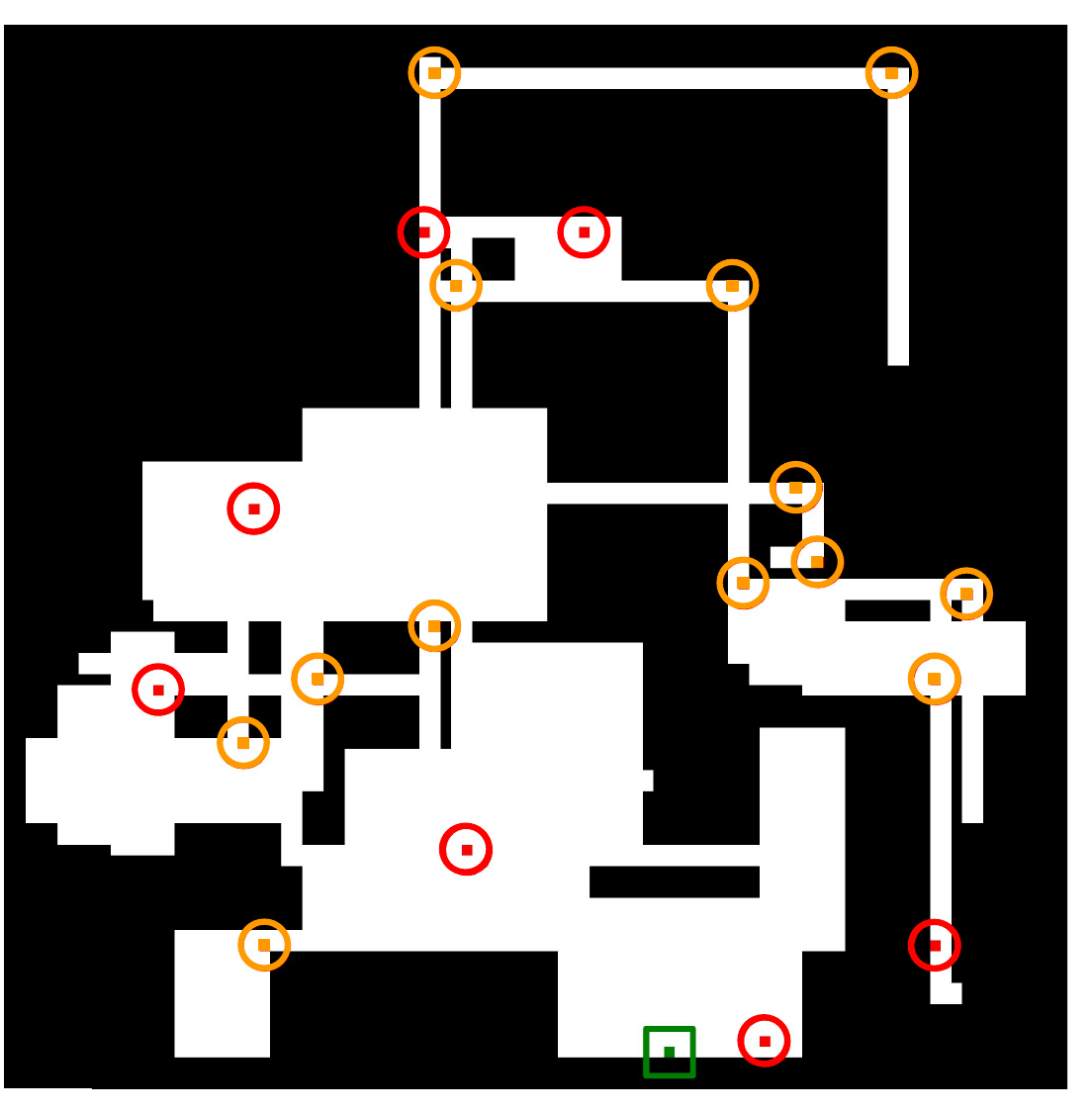}
    \caption{Sample DADENCE deployment: orange circles mark agents positioned at or adjacent to valid corners, red circles mark agents away from valid corners, and the green square indicates the deployment point.}
    \label{fig:dadence_valid}
\end{figure}

Another interesting observation from our tests is that DADENCE tends to place many agents at or immediately adjacent to valid corners, as illustrated in Figure~\ref{fig:dadence_valid}, where 13 out of 20 agents (65\%) occupy such locations. Due to the fact that DADENCE is the most agent-efficient algorithm across all successful trials, this behavior further supports the near-optimality of valid-corner placement. 
\subsection{Centralized versus Decentralized}
Having compared successful and unsuccessful algorithms, we now turn to how the centralized and decentralized methods relate to one another. Among the benchmarked algorithms, CADENCE, ISDA, and Lattices are centralized planners, while DADENCE, APF, VDN, and IDQL are decentralized.

Centralized planners benefit from global information and coordinated decision-making, which yields systematic strategies that reliably find solutions in the POCGAGP setting. The main differences between CADENCE, ISDA, and Lattices lie in their planning strategies. CADENCE’s strategy is the most effective: it solves the POCGAGP on all tested worlds, requires fewer steps than ISDA, and uses fewer agents than the lattice-based methods, which themselves do not always find a solution (Figs.~\ref{fig:data_steps}--\ref{fig:data_max}).

The decentralized algorithms present a sharp contrast. Most of them fail to consistently cover the world and are strongly dominated by the centralized methods in terms of success rate. DADENCE is the notable exception: it clearly separates itself from APF, VDN, and IDQL, reliably covering the worlds and competing with the best centralized approaches. This robustness stems from the fact that DADENCE explicitly enforces both coverage expansion and the maintenance of connectivity, properties that APF and the MARL-based methods lack. In this setting, when decentralized algorithms are carefully designed, agents can continually reorganize and sometimes discover more resource-efficient coverage configurations than centralized planners, as suggested by the final and peak agent counts in Figs.~\ref{fig:data_final} and \ref{fig:data_max}, albeit at the cost of additional time steps (Figure~\ref{fig:data_steps}). By contrast, APF, VDN, and IDQL illustrate how naive or purely learning-based decentralized policies are highly vulnerable to local minima, network disconnections, and poor generalization.

Although our POCGAGP formulation does not explicitly reward the structural benefits of decentralization, it is important to note that decentralized systems avoid dependence on a single controller and are thus more robust to controller failure. This advantage, while not reflected in our current evaluation, becomes critical in POCGAGP settings with intermittent communication, hardware faults, or adversarial disturbances.

\subsection{Impacts of deallocation}
As a final point in this section, we examine the role of deallocation in POCGAGP, noting that its benefits vary across algorithms in our experiments. This is reflected by comparing each algorithm’s final agent count in Figure~\ref{fig:data_final} with its peak count in Figure~\ref{fig:data_max}: the reduction from peak to final agents is entirely due to deallocation.

For agent-efficient methods such as DADENCE, this reduction is modest, since DADENCE already uses few unnecessary terminal agents. CADENCE, by contrast, gains substantially: as shown in Figure~\ref{fig:cadence_de_alloc}, without deallocation it would terminate with $N_{\max}$ agents (matching its peak usage), whereas with deallocation both the peak and final agent counts are much lower.

The benefit is even more pronounced for the square and triangle lattice algorithms. Naively, they heavily over-deploy agents, carpeting the environment with lattice structures. Deallocation allows them to release terminal agents and reposition them to new lattice locations, typically achieving coverage under the same $N_{\max}$ constraint. Without deallocation, these lattice-based methods would cover only a small fraction of the environment and would rarely solve the POCGAGP reliably, even on comparatively small instances.

Taken together, our results show that CADENCE and DADENCE outperform the state-of-the-art results for POCGAGP on the orthogonal plane, maintaining coverage and connectivity in the discrete setting. Both algorithms are able to deploy MANETs that reliably solve POCGAGP, but the choice between them depends on the objective: DADENCE is preferable when minimizing agent count or requiring decentralization, whereas CADENCE is preferable when minimizing the number of steps.

\section{Conclusion}
\label{conclusion}

In this paper, we present a comprehensive framework for self-forming Mobile Ad Hoc Networks (MANETs) by formalizing the Partially Observable Cooperative Guard Art Gallery Problem (POCGAGP), a model that requires agents to collectively achieve full coverage of initially unknown environments while strictly maintaining continuous connectivity. To address this fundamental challenge, we introduce two algorithms: CADENCE, a centralized distributed planner that leverages valid-corner geometry to provide theoretical worst-case guarantees and rapid coverage; and DADENCE, a fully decentralized scheme that maximizes agent efficiency and robustness through local coordination.

We rigorously validate these approaches against the benchmark suite of 1,500 orthogonal dungeon environments, utilizing a quadtree decomposition metric to quantify structural complexity. Our extensive empirical results demonstrate that while standard heuristics and learning-based methods are prone to local minima or scalability issues in complex topologies, our algorithms consistently achieve 100\% coverage.

Ultimately, this work highlights that integrating geometric abstractions into robotics policies enables scalable, guaranteed solutions for critical multi-agent problems. The POCGAGP framework lays the foundation for future research into more unstructured domains, including non-orthogonal worlds, three-dimensional volumes, and adversarial environments where communication resilience is important.

\section{Acknowledgments}
We would like to acknowledge multiple people who have supported this work, with special thanks to Professor Jay Weitzen, Chloe Si, Harley Wiltzer, Yanting Pan, Wilfred Mason, the members of the McGill MRL lab, and the members of Paris-Saclay CentraleSupélec L2S.

ChatGPT 5.1 and Grammarly were used in the grammar editing of the paper. All changes were verified and approved by a human author before the changes were made.

\bibliographystyle{ieeetr}

\bibliography{references}

\end{document}